\documentclass[twoside,11pt]{article}

\usepackage{jmlr2e,url}

\usepackage{psfrag,color}
\usepackage{amssymb,amsmath}

\newcommand{\mysec}[1]{Section~\ref{sec:#1}}
\newcommand{\eq}[1]{Eq.~(\ref{eq:#1})}
\newcommand{\myfig}[1]{Figure~\ref{fig:#1}}

\newcommand{\BEAS}{\begin{eqnarray*}}
\newcommand{\EEAS}{\end{eqnarray*}}
\newcommand{\BEA}{\begin{eqnarray}}
\newcommand{\EEA}{\end{eqnarray}}
\newcommand{\BEQ}{\begin{equation}}
\newcommand{\EEQ}{\end{equation}}
\newcommand{\BIT}{\begin{itemize}}
\newcommand{\EIT}{\end{itemize}}
\newcommand{\BNUM}{\begin{enumerate}}
\newcommand{\ENUM}{\end{enumerate}}

\newcommand{\diag}{\mathop{\bf diag}} 
\newcommand{\BA}{\begin{array}}
\newcommand{\EA}{\end{array}} 

\newcommand{\eg}{{\it e.g.}}

\newcommand{\ones}{\mathbf 1}

\newcommand{\reals}{{\mbox{\bf R}}}

\newcommand{\symm}{{\mbox{\bf S}}}  
\newcommand{\idm}{{\bf I}}

\newcommand{\Rank}{\mathop{\bf Rank}}
\newcommand{\Card}{\mathop{\bf Card}}
\newcommand{\Tr}{\mathop{\bf Tr}}
\newcommand{\lambdamax}{{\lambda_{\rm max}}}

\newcommand{\dsp}{\displaystyle}

\newcommand{\argmax}{\mathop{\rm argmax}}

\begin{document}

\title{Optimal Solutions for \\Sparse Principal Component Analysis}

\author{\name Alexandre d'Aspremont   \email aspremon@Princeton.edu \\
 \addr
ORFE, Princeton University, \\
Princeton, NJ 08544, USA.\\
 \AND
 \name
Francis Bach \email francis.bach@mines.org \\
\addr
INRIA - Willow project\\
D\'epartement d'Informatique, Ecole Normale Sup\'erieure\\
45, rue d'Ulm, 75230 Paris, France\\
\AND
\name
Laurent El Ghaoui \email elghaoui@eecs.Berkeley.edu \\
\addr
EECS Department, U.C. Berkeley,\\
Berkeley, CA 94720, USA.
}

\editor{}
\maketitle

\begin{abstract}
Given a sample covariance matrix, we examine the problem of maximizing the variance explained by a linear combination of the input variables while constraining the number of nonzero coefficients in this combination. This is known as sparse principal component analysis and has a wide array of applications in machine learning and engineering. We formulate a new semidefinite relaxation to this problem and derive a greedy algorithm that computes a \emph{full set} of good solutions for all target numbers of non zero coefficients, with total complexity $O(n^3)$, where $n$ is the number of variables. We then use the same relaxation to derive sufficient conditions for global optimality of a solution, which can be tested in $O(n^3)$ per pattern. We discuss applications in subset selection and sparse recovery and show on artificial examples and biological data that our algorithm does provide globally optimal solutions in many cases. 
\end{abstract}


\section{Introduction}
Principal component analysis (PCA) is a classic tool for data analysis, visualization or compression and has a wide range of applications throughout science and engineering. Starting from a multivariate data set, PCA finds linear combinations of the variables called \emph{principal components}, corresponding to orthogonal directions maximizing variance in the data.
Numerically, a full PCA involves a singular value decomposition of the data matrix.

One of the key shortcomings of PCA is that the factors are linear combinations of \emph{all} original variables; that is, most of factor coefficients (or loadings) are non-zero.  This means that while PCA facilitates model interpretation and visualization by concentrating the information in a few factors, the factors themselves are still constructed using all variables, hence are often hard to interpret.

In many applications, the coordinate axes involved in the factors have a direct physical interpretation. In financial or  biological applications, each axis might correspond to a specific asset or gene. In problems such as these, it is natural to seek a trade-off between the two goals of \emph{statistical fidelity} (explaining most of the variance in the data) and \emph{interpretability} (making sure that the factors involve only a few coordinate axes).  Solutions that have only a few nonzero coefficients in the principal components are usually easier to interpret. Moreover, in some applications, nonzero coefficients have a direct cost (\eg, transaction costs in finance) hence there may be a direct trade-off between statistical fidelity and practicality. Our aim here is to efficiently derive \emph{sparse principal components}, i.e, a set of sparse vectors that explain a maximum amount of  variance. Our belief is that in many applications, the decrease  in statistical fidelity required to obtain sparse factors is  small and relatively benign. 

In what follows, we will focus on the problem of finding sparse factors which explain a maximum amount of variance, which can be written:
\BEQ \label{eq:spca-intro}
\max_{ \| z \|  \leq 1 } z^T \Sigma z - \rho \Card(z)
\EEQ
in the variable $z\in\reals^n$, where $\Sigma\in\symm_n$ is the (symmetric positive semi-definite) sample covariance matrix, $\rho$ is a parameter controlling sparsity, and $\Card(z)$ denotes the cardinal (or $\ell_0$ norm) of $z$, i.e. the number of non zero coefficients of $z$.

While PCA is numerically easy, each factor requires computing a leading eigenvector, which can be done in $O(n^2)$, sparse PCA is a hard combinatorial problem. In fact, \citet{sparseLDA} show that the subset selection problem for ordinary least squares, which is NP-hard~\citep{nphard-subsetselection}, can be reduced to a sparse generalized eigenvalue problem, of which sparse PCA is a particular intance. Sometimes ad hoc ``rotation'' techniques are used to post-process the results from PCA and find interpretable directions underlying a particular subspace (see \citet{Joll95}). Another simple solution is to \emph{threshold} the loadings with small absolute value to zero \citep{cadi95}. A more systematic approach to the problem arose in recent years, with various researchers proposing nonconvex algorithms  (e.g., SCoTLASS by \citet{Joll03}, SLRA by \citet{Zhan02a} or D.C. based methods \citep{Srip07} which find modified principal components  with zero loadings. The SPCA algorithm, which is based on the representation of PCA as a regression-type optimization problem~\citep{Zou04}, allows the application of the LASSO \citep{tibs96}, a penalization technique based on the $\ell_1$ norm. With the exception of simple thresholding, all the algorithms above require solving non convex problems. Recently also, \citet{dasp04a} derived an $\ell_1$ based semidefinite relaxation for the sparse PCA problem (\ref{eq:spca-intro}) with a complexity of $O(n^{4}\sqrt{\log n})$ for a given $\rho$. Finally, \citet{Mogh06b} used greedy search and branch-and-bound methods to solve small instances of problem (\ref{eq:spca-intro}) exactly and get good solutions for larger ones. Each step of this greedy algorithm has complexity $O(n^3)$, leading to a total complexity of $O(n^4)$ for a full set of solutions.

Our contribution here is twofold. We first derive a greedy algorithm for computing a \emph{full set} of good solutions (one for each target sparsity between 1 and $n$) at a total numerical cost of $O(n^3)$ based on the convexity of the of the largest eigenvalue of
a symmetric matrix. We then derive \emph{tractable} sufficient conditions for a vector $z$ to be a \emph{global} optimum of (\ref{eq:spca-intro}). This means in practice that, given a vector $z$ with support $I$, we can test if $z$ is a globally optimal solution to problem (\ref{eq:spca-intro}) by performing a few binary search iterations to solve a one dimensional convex minimization problem. In fact, we can take any sparsity pattern candidate from any algorithm and test its optimality. This paper builds on the earlier conference version~\citep{fullpathPCA}, providing new and simpler conditions for optimality and describing applications to subset selection and sparse recovery.

While there is certainly a case to be made for $\ell_1$ penalized maximum eigenvalues (\`a la \citet{dasp04a}), we strictly focus here on the $\ell_0$ formulation. However, it was shown recently (see \citet{Cand05}, \cite{Dono05} or \cite{Mein06a} among others) that there is in fact a deep connection between $\ell_0$ constrained extremal eigenvalues and LASSO type variable selection algorithms. Sufficient conditions based on sparse eigenvalues (also called restricted isometry constants in \citet{Cand05}) guarantee consistent variable selection (in the LASSO case) or sparse recovery (in the decoding problem). The results we derive here produce upper bounds on sparse extremal eigenvalues and can thus be used to prove consistency in LASSO estimation, prove perfect recovery in sparse recovery problems, or prove that a particular solution of the subset selection problem is optimal. Of course, our conditions are only sufficient, not necessary (which would contradict the NP-Hardness of subset selection) and the duality bounds we produce on sparse extremal eigenvalues cannot always be tight, but we observe that the duality gap is often small. 

The paper is organized as follows. We begin by formulating the sparse PCA problem in Section \ref{sec:spca}. In Section \ref{sec:greedy}, we write an efficient algorithm for computing a full set of candidate solutions to problem (\ref{eq:spca-intro}) with total complexity $O(n^3)$. In \mysec{semidefinite} we then formulate a convex relaxation for the sparse PCA problem, which we use in Section \ref{sec:tightness} to derive tractable sufficient conditions for the global optimality of a particular sparsity pattern. In Section \ref{sec:apps} we detail applications to subset selection, sparse recovery and variable selection. Finally, in Section \ref{sec:numer}, we test the numerical performance of these results. 

\subsection*{Notation}
For a vector $z\in\reals$, we let $\|z\|_1 = \sum_{i=1}^n |z_i|$ and $\|z\| = \left( \sum_{i=1}^n z_i^2 \right)^{1/2}$, $\Card(z)$ is the cardinality of $z$, i.e. the number of nonzero coefficients of $z$, while the support $I$ of $z$ is the set $\{i:~z_i\neq 0\}$ and we use $I^c$ to denote its complement. For $\beta\in\reals$, we write $\beta_+ = \max \{\beta,0\}$ and for $X\in\symm_n$ (the set of symmetric matrix of size $n\times n$) with eigenvalues $\lambda_i$, $\Tr(X)_+=\sum_{i=1}^n \max\{\lambda_i,0\}$. The vector of all ones is written $\ones$, while the identity matrix is written $\idm$. The diagonal matrix with the vector $u$ on the diagonal is written $\diag(u)$.

\section{Sparse PCA}
\label{sec:spca} Let $\Sigma\in\symm_n$ be a symmetric matrix. We consider the following sparse PCA problem:
\BEQ \label{eq:pca-card}
\phi(\rho) \equiv \max_{ \| z \|  \leq 1 } z^T \Sigma z - \rho \Card(z)
\EEQ
in the variable $z\in\reals^n$ where $\rho >0$ is a parameter controlling sparsity. We assume without loss of generality that $\Sigma\in\symm_n$ is positive semidefinite and that the $n$ variables are ordered by decreasing marginal variances, i.e.
that $\Sigma_{11} \geq \ldots \geq \Sigma_{nn}$. We also assume that we are given a square root $A$ of the matrix $\Sigma$ with $ \Sigma = A^T A$, where $A\in\reals^{n \times n}$ and we denote by $a_1,\dots,a_n \in \reals^n$ the columns of $A$. Note that the problem and our algorithms are invariant by permutations of $\Sigma$ and by the choice of square root $A$. In practice, we are very often given the data matrix $A$ instead of the covariance $\Sigma$.

A problem that is directly related to (\ref{eq:pca-card}) is that of computing a cardinality constrained maximum eigenvalue, by solving:
\BEQ\label{eq:card-eig}
\BA{ll}
\mbox{maximize} & z^T \Sigma z\\
\mbox{subject to} & \Card(z) \leq k\\
& \|z\|=1,
\EA\EEQ
in the variable $z\in\reals^n$. Of course, this problem and (\ref{eq:pca-card}) are related. By duality, an upper bound on the optimal value of (\ref{eq:card-eig}) is given by:
\[
\inf_{\rho \in P} \phi(\rho) + \rho k.
\]
where $P$ is the set of penalty values for which $\phi(\rho)$ has been computed. This means in particular that if a point $z$ is provably optimal for (\ref{eq:pca-card}), it is also globally optimum for (\ref{eq:card-eig}) with $k=\Card(z)$.

We now begin by reformulating (\ref{eq:pca-card}) as a relatively simple convex maximization problem. Suppose that $\rho \geq \Sigma_{11}$. Since $z^T\Sigma z\leq \Sigma_{11}(\sum_{i=1}^n |z_i| )^2$ and $(\sum_{i=1}^n |z_i| )^2\leq \|z\|^2\Card(z)$ for all $z\in\reals^n$, we always have:
\[\BA{ll}
\phi(\rho) &= \max_{ \| z \|  \leq  1 } z^T \Sigma z - \rho \Card(z) \\
&\leq (\Sigma_{11}-\rho) \Card(z)\\  
&\leq 0,
\EA\]
hence the optimal solution to (\ref{eq:pca-card}) when $\rho \geq \Sigma_{11}$ is $z=0$. From now on, we assume $\rho \leq \Sigma_{11}$ in which case the inequality $\|z\|\leq1$ is tight. We can represent the sparsity pattern of a vector $z$ by a vector $u \in \{0,1\}^n$ and rewrite (\ref{eq:pca-card}) in the equivalent form:
\[\BA{ll}
\phi(\rho) & = \dsp \max_{u \in \{0,1\}^n } \lambdamax(\diag(u)\Sigma\diag(u)) - \rho \ones^Tu\\
& = \dsp \max_{u \in \{0,1\}^n } \lambdamax(\diag(u)A^TA\diag(u)) - \rho \ones^Tu\\
& = \dsp \max_{u \in \{0,1\}^n } \lambdamax(A\diag(u)A^T) - \rho \ones^Tu,
\EA\]
using the fact that $\diag(u)^2=\diag(u)$ for all variables $u \in \{0,1\}^n $
and that for any matrix $B$, $\lambdamax(B^T B) = \lambdamax(BB^T)$. We then have:
\[\BA{ll}
\phi(\rho) & = \dsp \max_{u \in \{0,1\}^n } \lambdamax(A\diag(u)A^T) - \rho \ones^Tu\\
& = \dsp \max_{ \| x \| = 1}~\max_{u \in \{0,1\}^n} x^TA\diag(u)A^Tx- \rho \ones^Tu\\
& = \dsp \max_{ \| x \| = 1}~\max_{u \in \{0,1\}^n}  \sum_{i=1}^n
u_i( (a_i^T x)^2 - \rho ).
\EA\]
Hence we finally get, after maximizing in $u$ (and using $\max_{v \in \{0,1\}} \beta v = \beta_+$):
\BEQ \label{eq:pca-ncvx}
\phi(\rho) = \max_{ \| x \| = 1} \sum_{i=1}^n((a_i^Tx)^2-\rho)_+,
\EEQ
which is a nonconvex problem in the variable $x\in\reals^n$. We then select variables $i$ such that $(a_i^Tx)^2-\rho>0$. Note that if $\Sigma_{ii} = a_i^Ta_i < \rho$, we must have $(a_i^Tx)^2 \leq \|a_i\|^2 \|x\|^2 < \rho$ hence variable $i$ will never be part of the optimal subset and we can remove it.

\section{Greedy Solutions}
\label{sec:greedy} In this section, we focus on finding a good solution to problem (\ref{eq:pca-card}) using greedy methods. We first present very simple preprocessing solutions with complexity $O(n\log n)$ and $O(n^2)$. We then recall a simple greedy algorithm with complexity $O(n^4)$. Finally, our first contribution in this section is to derive an approximate greedy algorithm that computes a full set of (approximate) solutions for problem (\ref{eq:pca-card}), with total complexity $O(n^3)$.

\subsection{Sorting and Thresholding}
The simplest ranking algorithm is to sort the diagonal of the matrix $\Sigma$ and rank the variables by variance. This works intuitively because the diagonal is a rough proxy for the eigenvalues: the Schur-Horn theorem states that the diagonal of a matrix majorizes its  eigenvalues \citep{Horn85}; sorting costs $O(n\log n)$. Another quick solution is to compute the leading eigenvector of $\Sigma$ and form a sparse vector by thresholding to zero the coefficients whose magnitude is smaller than a certain level. This can be done with cost $O(n^2)$.

\subsection{Full greedy solution}
\label{sec:greedy-classic} Following \citet{Mogh06b}, starting from an initial solution of cardinality one at $\rho=\Sigma_{11}$, we can update an increasing sequence of index sets $I_k\subseteq[1,n]$, scanning all the remaining variables to find the index with maximum variance contribution. 
\\ 
\line(1,0){432}
\vskip 0ex
\textbf{Greedy Search Algorithm.}
\BIT
\item {\bf Input}: $\Sigma \in \reals^{n \times n}$
\item {\bf Algorithm}:
\BNUM
\item Preprocessing: sort variables by decreasing diagonal elements and permute elements of $\Sigma$ accordingly. Compute the Cholesky decomposition $\Sigma = A^T A $.
\item Initialization: $I_1 = \{1\}$,  $x_1=a_1/\|a_1\|$.
\item Compute $i_k = \argmax_{i \notin I_k}  \lambdamax\left(
\sum_{j\in I_k \cup \{ i \} } a_j a_j^T \right)$.
\item Set $I_{k+1}=I_k\cup\{i_k\}$ and compute $x_{k+1}$ as the leading eigenvector of $\sum_{ j \in I_{k+1}}a_ja_j^T$.
\item Set $k=k+1$. If $k<n$ go back to step 3.
\ENUM
\item {\bf Output}: sparsity patterns $I_k$.
\EIT
\vskip -2ex
\line(1,0){432}
\vskip 0ex
At every step, $I_k$ represents the set of nonzero elements (or sparsity pattern) of the current point and we can define $z_k$ as the solution to problem (\ref{eq:pca-card}) given $I_k$, which is:
\[
z_k=\argmax_{ \{z_{I_k^c}=0,~\|z\|=1 \}} z^T \Sigma z - \rho k,\\
\]
which means that $z_k$ is formed by padding zeros to the leading eigenvector of the submatrix $\Sigma_{I_k,I_k}$. Note that the entire algorithm can be written in terms of a factorization $\Sigma = A^T A$ of the matrix $\Sigma$, which means significant computational savings when $\Sigma$ is given as a Gram matrix. The matrices $\Sigma_{I_k,I_k}$ and $\sum_{i \in I_k} a_i a_i^T$ have the same eigenvalues and their eigenvectors are transformed of each other through the matrix $A$, i.e., if $z$ is an eigenvector of $\Sigma_{I_k,I_k}$, then $ A_{I_k} z / \|A_{I_k} z \|$ is an eigenvector of $A_{I_k} A_{I_k}^T$.

\subsection{Approximate greedy solution}
\label{sec:greedy-path} Computing $n-k$ eigenvalues at each iteration is costly and we can use the fact that $uu^T$ is a subgradient of $\lambdamax$ at $X$ if $u$ is a leading eigenvector of $X$~\citep{Boyd03}, to get: 
\BEQ
\label{eq:lambdamax-lower-bound}
\lambdamax\left(
\sum_{j\in I_k \cup \{ i \} } a_j a_j^T \right)
\geq 
\lambdamax\left(
\sum_{j\in I_k } a_j a_j^T \right)
+ (x_k^T a_i)^2,
\EEQ
which means that the variance is increasing by at least $(x_k^T a_i)^2$ when variable $i$ is added to $I_k$. This provides  a lower bound on the objective which does not require finding $n-k$ eigenvalues at each iteration. We then derive the following algorithm:\\
\line(1,0){432}
\vskip 0ex
\textbf{Approximate Greedy Search Algorithm.}
\BIT
\item {\bf Input}: $\Sigma \in \reals^{n \times n}$
\item {\bf Algorithm}:
\BNUM
\item Preprocessing. Sort variables by decreasing diagonal elements and permute elements of $\Sigma$ accordingly. Compute the Cholesky decomposition $\Sigma = A^T A $.
\item Initialization: $I_1 = \{1\}$,  $x_1=a_1/\|a_1\|$.
\item Compute $i_k = \argmax_{i \notin I_k}  (x_k^T a_i)^2 $
\item Set $I_{k+1}=I_k\cup\{i_k\}$ and compute $x_{k+1}$ as the leading eigenvector of $\sum_{ j \in I_{k+1}}a_j a_j^T$.
\item Set $k=k+1$. If $k<n$ go back to step 3.
\ENUM
\item {\bf Output}: sparsity patterns $I_k$.
\EIT
\vskip -2ex
\line(1,0){432}
\vskip 0ex
Again, at every step, $I_k$ represents the set of nonzero elements (or sparsity pattern) of the current point and we can define $z_k$ as the solution to problem (\ref{eq:pca-card}) given $I_k$, which is:
\[
z_k=\argmax_{ \{z_{I_k^c}=0,~\|z\|=1 \}} z^T \Sigma z - \rho k,\\
\]
which means that $z_k$ is formed by padding zeros to the leading eigenvector of the submatrix $\Sigma_{I_k,I_k}$.  Better points can be found by testing the variables corresponding to the $p$ largest values of $(x_k^T a_i)^2$ instead of picking only the best one.

\subsection{Computational Complexity}
The complexity of computing a greedy regularization path using the classic greedy algorithm in \mysec{greedy-classic} is $O(n^4)$: at each step $k$, it computes $(n-k)$ maximum eigenvalue of matrices with size $k$. The approximate algorithm in \mysec{greedy-path} computes a full path in $O(n^3)$: the first Cholesky decomposition is $O(n^3)$, while the complexity of the $k$-th iteration is $O(k^2)$ for the maximum eigenvalue problem and $O(n^2)$ for computing all products $(x^T a_j)$. Also, when the matrix $\Sigma$ is directly given as a Gram matrix $A^T A$ with $A\in\reals^{q \times n}$ with $q<n$, it is advantageous to use $A$ directly as the square root of $\Sigma$ and the total complexity of getting the path up to cardinality $p$ is then reduced to $O(p^3 + p^2 n)$ (which is $O(p^3)$ for the eigenvalue problems and $O(p^2 n )$ for computing the vector products).

\section{Convex Relaxation}
\label{sec:semidefinite} In \mysec{spca}, we showed that the original sparse PCA problem (\ref{eq:pca-card}) could also be written as in (\ref{eq:pca-ncvx}):
\[
\phi(\rho) = \max_{ \| x \| = 1} \sum_{i=1}^n((a_i^Tx)^2-\rho)_+.
\]
Because the variable $x$ appears solely through $X=xx^T$, we can reformulate the problem in terms of $X$ only, using the fact that when $\|x\|=1$, $X=xx^T$ is equivalent to $\Tr(X)=1$, $X\succeq 0$ and $\Rank(X)=1$. We thus  rewrite (\ref{eq:pca-ncvx}) as:
\[\BA{lll}
\phi(\rho) =&\mbox{max.}&  \sum_{i=1}^n(a_i^TXa_i-\rho)_+ \\
&\mbox{s.t.} & \Tr(X)=1,~\Rank(X)=1\\
& & X\succeq 0.\\
\EA\]
Note that because we are maximizing a convex function over the convex set (spectahedron) $\Delta_n=\{X\in\symm_n:~\Tr(X)=1,~ X\succeq 0\}$, the solution must be an extreme point of $\Delta_n$ (i.e. a rank one matrix), hence we can drop the rank constraint here. Unfortunately, $X \mapsto (a_i^TXa_i-\rho)_+$, the function we are \emph{maximizing}, is convex in $X$ and not concave, which means that the above problem is still hard. However, we show below that on rank one elements of $\Delta_n$, it is also equal to a concave function of $X$, and we use this to produce a semidefinite relaxation of problem (\ref{eq:pca-card}).

\begin{proposition} \label{prop:relax} Let $A\in\reals^{n \times n}$, $\rho \geq 0$ and denote by $a_1,\dots,a_n \in \reals^n$ the columns of $A$, an upper bound on:
\BEQ\label{eq:pca-plus}\BA{lll}
\phi(\rho) =&\mbox{max.}&  \sum_{i=1}^n(a_i^TXa_i-\rho)_+ \\
&\mbox{s.t.} & \Tr(X)=1,~ X\succeq 0, ~\Rank(X)=1\EA\EEQ
can be computed by solving
\BEQ  \label{eq:pca-relax}
\BA{lll}
\psi(\rho) =&\mbox{max.}&  \sum_{i=1}^n \Tr(X^{1/2}B_i X^{1/2})_+ \\
&\mbox{s.t.} & \Tr(X)=1,~X\succeq 0.
\EA\EEQ
in the variables $X\in\symm_n$, where $B_i=a_ia_i^T - \rho\idm$, or also:
\BEQ \label{eq:primal-p}
\BA{lll}
\psi(\rho) = &\mbox{max.}&  \sum_{i=1}^n\Tr(P_iB_i) \\
&\mbox{s.t.} & \Tr(X)=1,~X\succeq 0,~X\succeq P_i \succeq 0,
\EA\EEQ 
which is a semidefinite program in the variables $X\in\symm_n,~P_i\in\symm_n$.
\end{proposition}
\begin{proof}We let $X^{1/2}$ denote the positive square root (i.e. with nonnegative eigenvalues) of a symmetric positive semi-definite matrix $X$. In particular, if $X=xx^T$ with $\|x\|=1$, then $X^{1/2}= X = xx^T$, and for all  $\beta\in \reals$, $\beta xx^T$ has one eigenvalue equal to $\beta$ and $n-1$ equal to 0, which implies $\Tr(\beta xx^T)_+=\beta_+$. We thus get:  
\BEAS
 (a_i^TXa_i-\rho)_+ & = & \Tr ( (a_i^T xx^T a_i - \rho) xx^T )_+
 \\
& = &  \Tr(x ( x^T a_ia_i^T x -\rho ) x^T )_+ \\
& = &  \Tr( X^{1/2} a_i a_i^T X^{1/2} - \rho X  )_+
= \Tr( X^{1/2} (  a_i a_i^T -\rho \idm)  X^{1/2}    )_+.
\EEAS
For any symmetric matrix $B$, the function $X \mapsto \Tr( X^{1/2} B X^{1/2} )_+$ is concave on the set of symmetric positive semidefinite matrices, because we can write it as:
\[\BA{ll}
\Tr(X^{1/2}BX^{1/2})_+&=\dsp\max_{\{0\preceq P \preceq X\}} \Tr(PB)\\
&=\dsp\min_{\{Y\succeq B,~Y\succeq 0\}} \Tr(YX),
\EA\]
where this last expression is a concave function of $X$ as a pointwise minimum of affine functions. We can now relax the original problem into a convex optimization problem by simply dropping the rank constraint, to get:
$$ \BA{lll}
\psi(\rho) \equiv &\mbox{max.}&  \sum_{i=1}^n\Tr(X^{1/2}a_ia_i^TX^{1/2}-\rho X)_+ \\
&\mbox{s.t.} & \Tr(X)=1,~X\succeq 0,
\EA
$$
which is a convex program in $X\in\symm_n$. Note that because $B_i$   has at most one nonnegative eigenvalue, we can replace $\Tr(X^{1/2}a_ia_i^TX^{1/2}-\rho X)_+$ by $\lambdamax(X^{1/2}a_ia_i^TX^{1/2}-\rho X)_+$ in the above program. Using the representation of $\Tr(X^{1/2}BX^{1/2})_+$ detailed above, problem (\ref{eq:pca-relax}) can be written as a semidefinite program:
\[
\BA{lll}
\psi(\rho) = &\mbox{max.}&  \sum_{i=1}^n\Tr(P_iB_i) \\
&\mbox{s.t.} & \Tr(X)=1,~X\succeq 0,~X\succeq P_i \succeq 0,
\EA\]
in the variables $X\in\symm_n,~P_i\in\symm_n$, which is the desired result. 
\end{proof}

Note that we always have $\psi(\rho)\geq\phi(\rho)$ and when the solution to the above semidefinite program has rank one, $\psi(\rho)=\phi(\rho)$ and the semidefinite relaxation (\ref{eq:primal-p}) is \emph{tight}. This simple fact allows us to derive sufficient global optimality conditions for the original sparse PCA problem.

\section{Optimality Conditions}
\label{sec:tightness} In this section, we derive necessary and sufficient conditions to test the optimality of solutions to the relaxations obtained in Sections~\ref{sec:greedy}, as well as sufficient condition for the tightness of the semidefinite relaxation in (\ref{eq:primal-p}).

\subsection{Dual problem and optimality conditions}
We first derive the dual problem to (\ref{eq:primal-p}) as well as the 
Karush-Kuhn-Tucker (KKT) optimality conditions:

\begin{lemma}
\label{lemma:dual}
Let $A\in\reals^{n \times n}$, $\rho \geq 0$ and denote by $a_1,\dots,a_n \in \reals^n$ the columns of $A$. The dual of problem  (\ref{eq:primal-p}): 
\[
\BA{lll}
\psi(\rho) = &\mbox{max.}&  \sum_{i=1}^n\Tr(P_iB_i) \\
&\mbox{s.t.} & \Tr(X)=1,~X\succeq 0,~X\succeq P_i \succeq 0,
\EA\]
in the variables $X\in\symm_n,~P_i\in\symm_n$, is given by:
\BEQ\label{eq:dual-y}
\BA{ll}
\mbox{min.} & \lambdamax\left(\sum_{i=1}^n Y_i\right)\\
\mbox{s.t.} & Y_i\succeq B_i,~Y_i\succeq0,\quad i=1,\ldots,n.
\EA\EEQ
in the variables $Y_i \in \symm_n$. Furthermore, the KKT optimality conditions for this pair of semidefinite programs are given by:
\BEQ
 \label{eq:kkt}
\left\{\BA{l}
\left(\sum_{i=1}^n Y_i\right)X=\lambdamax\left(\sum_{i=1}^n Y_i\right)X\\
(X-P_i)Y_i=0,~P_iB_i=P_iY_i\\
Y_i\succeq B_i,~Y_i,X,P_i \succeq 0,~ X \succeq P_i,~\Tr X = 1.
\EA\right.\EEQ
\end{lemma}
\begin{proof}
Starting from:
\[ \BA{ll}
\mbox{max.}&  \sum_{i=1}^n \Tr(P_iB_i) \\
\mbox{s.t.} & 0\preceq P_i \preceq X\\
& \Tr(X)=1,~X\succeq 0,
\EA\]
we can form the Lagrangian as:
\[
L(X,P_i,Y_i)=\sum_{i=1}^n \Tr(P_iB_i)+\Tr(Y_i(X-P_i))
\]
in the variables $X,P_i,Y_i\in\symm_n$, with $X,P_i,Y_i \succeq 0$ and $\Tr(X)=1$.
Maximizing $L(X,P_i,Y_i)$ in the primal variables $X$ and $P_i$ leads to problem (\ref{eq:dual-y}). The KKT conditions for this primal-dual pair of SDP can be derived from \citet[p.267]{Boyd03}.

\end{proof}

\subsection{Optimality conditions for rank one solutions}
\label{sec:semidef-optim} We now derive the KKT conditions for problem (\ref{eq:primal-p}) for the particular case where we are given a rank one candidate solution $X=xx^T$ and need to test its optimality. These necessary and sufficient conditions for the optimality of $X=xx^T$ for the convex relaxation then provide sufficient conditions for \emph{global} optimality for the non-convex problem (\ref{eq:pca-card}).

\begin{lemma}\label{lem:kkt-rank-one}
Let $A\in\reals^{n \times n}$, $\rho \geq 0$ and denote by $a_1,\dots,a_n \in \reals^n$ the columns of $A$. The rank one matrix $X=xx^T$ is an optimal solution of (\ref{eq:primal-p}) 
if and only if there are matrices $Y_i\in\symm_n,~i=1,\ldots,n$ such that:
\BEQ \label{eq:kkt-rank-one} \left\{\BA{l}
\lambdamax\left(\sum_{i=1}^n Y_i\right) = \sum_{i\in I } ( (a_i^Tx)^2-\rho )\\
x^TY_ix=\left\{\BA{l}
(a_i^Tx)^2-\rho ~\mbox{if } i \in I\\
0~\mbox{if } i \in I^c
\EA\right.\\
Y_i\succeq B_i,~Y_i \succeq 0.
\EA\right.\EEQ
where $B_i=a_ia_i^T-\rho \idm,~i=1,\ldots,n$ and $I^c$ is the complement of the set $I$ defined by:
\[
\max_{i \notin I} (a_i^T x)^2 \leq \rho \leq  \min_{i \in I} (a_i^T x)^2.
\]
Furthermore,  $x$ must be a leading eigenvector of both $\sum_{i \in I} a_i a_i^T$ and $\sum_{i=1}^n Y_i$.
\end{lemma}
\begin{proof}
We apply Lemma~\ref{lemma:dual} given $X=xx^T$. The condition $0 \preceq P_i \preceq xx^T$ is equivalent to $P_i = \alpha_i xx^T$ and $\alpha_i \in [0,1]$. The equation $P_iB_i=XY_i$ is then equivalent to $\alpha_i ( x^T B_i x - x^T Y_i x)=0$, with $x^T B_i x=(a_i^Tx)^2-\rho$ and the condition $(X-P_i)Y_i=0$ becomes $x^T Y_i x ( 1- \alpha_i)=0$. This means that $x^T Y_i x=((a_i^Tx)^2-\rho)_+$ and the first-order condition in (\ref{eq:kkt}) becomes $\lambdamax\left(\sum_{i=1}^n Y_i\right)=x^T\left(\sum_{i=1}^n Y_i\right)x$. Finally, we recall from \mysec{spca} that: 
\[\BA{ll}
\sum_{i\in I } ( (a_i^Tx)^2-\rho )&=\dsp \max_{ \| x \| = 1}~\max_{u \in \{0,1\}^n}  \sum_{i=1}^n u_i( (a_i^T x)^2 - \rho )\\
&=\dsp \max_{u \in \{0,1\}^n } \lambdamax(A\diag(u)A^T) - \rho \ones^Tu
\EA\]
hence $x$ must also be a leading eigenvector of $\sum_{i \in I} a_i a_i^T$.
\end{proof}

The previous lemma shows that given a candidate vector $x$, we can test the optimality
of $X=xx^T$ for the semidefinite program (\ref{eq:pca-relax}) by solving a semidefinite feasibility problem in the variables $Y_i \in \symm_n$. If this (rank one) solution $xx^T$ is indeed optimal for the semidefinite relaxation, then $x$ must also be \emph{globally} optimal for the original nonconvex combinatorial problem in (\ref{eq:pca-card}), so the above lemma provides sufficient global optimality conditions for the combinatorial problem (\ref{eq:pca-card}) based on the (necessary and sufficient) optimality conditions for the convex relaxation (\ref{eq:pca-relax}) given in lemma \ref{lemma:dual}. In practice, we are only given a sparsity pattern $I$ (using the results of \mysec{greedy} for example) rather than the vector $x$, but Lemma \ref{lem:kkt-rank-one} also shows that given $I$, we can get the vector $x$ as the leading eigenvector of $\sum_{i \in I} a_i a_i^T$.

The next result provides more refined conditions under which such a pair $(I,x)$ is optimal for some value of the penalty $\rho>0$ based on a local optimality argument. In particular, they allow us to fully specify the dual variables $Y_i$ for $i \in I$.

\begin{proposition}\label{prop:opt-cond}
Let $A\in\reals^{n \times n}$, $\rho \geq 0$ and denote by $a_1,\dots,a_n \in \reals^n$ the columns of $A$. Let $x$ be the largest eigenvector of $\sum_{i \in I} a_i a_i^T$. 
Let $I$ be such that:  \BEQ
\label{eq:rho-interval}
\max_{i \notin I} (a_i^T x)^2 < \rho <  \min_{i \in I} (a_i^T x)^2,
\EEQ
the matrix $X=xx^T$ is optimal  for problem (\ref{eq:primal-p}) if and only if there are matrices $Y_i\in\symm^n$ satisfying
\BEQ \label{eq:so-condition3}
\lambdamax \left( \sum_{i \in I} \frac{ B_i x x^T B_i }{x^T B_i x} +   \sum_{i \in I^c} Y_i \right) \leq \sum_{i \in I}((a_i^T x)^2 - \rho),
\EEQ
with $Y_i\succeq B_i - \frac{ B_i xx^T B_i} { x^T B_i x } ,~Y_i\succeq0$, where $B_i=a_ia_i^T-\rho \idm,~i=1,\ldots,n$.
\end{proposition}
\begin{proof}
We first prove the necessary condition by computing a first order expansion of the functions $F_i: X \mapsto \Tr (X^{1/2} B_i X^{1/2})_+$ around $X=xx^T$. The expansion is based on the results in Appendix~\ref{app:perturbations} which show how to compute derivatives of eigenvalues and projections on eigensubspaces. More precisely, Lemma~\ref{lem:expansion} states that if $x^T B x >0$, then, for any $Y \succeq 0$: 
$$
F_i((1-t) xx^T + t Y ) = F_i(xx^T) + \frac{t}{x^T B_i x} \Tr B_i xx^T B_i ( Y - xx^T) + O(t^{3/2}),
$$
while if $x^T B x < 0 $, then, for any $Y \succeq 0$,:
$$
F_i((1-t) xx^T + t Y ) = t_+
 \Tr \left(Y^{1/2} \left( B_i - \frac{ B_i x  x^T B_i}{ x^T B_i x} \right) Y^{1/2} \right)_+ 
+  O(t^{3/2}).
$$
Thus if $X=xx^T$ is a global maximum of $\sum_i F_i(X)$, then this first order expansion must reflect the fact that it is also local maximum, i.e. for all $Y \in \symm^n$ such that $Y \succeq 0$ and $\Tr Y = 1$, we must have:
\[
\lim_{t \to 0_+} \frac{1}{t}  \sum_{i=1}^n [ F_i((1-t) xx^T +tY )-F_i(xx^T) ] \leq 0,
\] 
which is equivalent to:
$$ -\sum_{i \in I} x^T B_i x +
\Tr Y \left(\sum_{i \in I}
\frac{B_i xx^T B_i}{x^T B_i x} \right) + \sum_{i \in I^c}
\Tr \left(Y^{1/2} \left( B_i - \frac{ B_i x  x^T B_i}{ x^T B_i x}  \right) Y^{1/2} \right)_+  \leq 0.
$$
Thus if $X=xx^T$ is optimal, with $\sigma = \sum_{i \in I} x^T B_i x$, we get:
\[ 
\displaystyle \max_{Y \succeq 0, \Tr Y=1} \Tr Y \left(
\sum_{i \in I} \frac{ B_i x x^T B_i }{x^T B_i x} - \sigma \idm  \right) + \sum_{i \in I^c} \Tr\left ( Y^{1/2} \left( B_i - B_i x ( x^T B_i x )^\dagger x^T B_i\right)Y^{1/2} \right)_+ \leq 0
\]
which is also in dual form (using the same techniques as in the proof of Proposition~\ref{prop:relax}):
\[
\displaystyle 
\min_{\{Y_i \succeq B_i  -  \frac{   B_i xx^T B_i   }{x^T B_i x}  , Y_i \succeq 0\}} \lambdamax
\left( \sum_{i \in I} \frac{ B_i x x^T B_i }{x^T B_i x} +   \sum_{i \in I^c} Y_i \right) \leq \sigma  ,
\]
which leads to the necessary condition. In order to prove sufficiency, the only non trivial condition to check in Lemma~\ref{lem:kkt-rank-one} is that $x^T Y_i x = 0$ for $i \in I^c$, which is a consequence of the inequality:
\[
x^T \left( \sum_{i \in I} \frac{ B_i x x^T B_i }{x^T B_i x} +   \sum_{i \in I^c} Y_i \right) x 
\leq \lambdamax
 \left( \sum_{i \in I} \frac{ B_i x x^T B_i }{x^T B_i x} +   \sum_{i \in I^c} Y_i \right)
 \leq x^T \left( \sum_{i \in I} \frac{ B_i x x^T B_i }{x^T B_i x}  \right) x.
\]
This concludes the proof.
\end{proof}
The original optimality conditions in (\ref{lem:kkt-rank-one}) are highly degenerate in $Y_i$ and this result refines these optimality conditions by invoking the local structure. The local optimality analysis in proposition \ref{prop:opt-cond} gives more specific constraints on the dual variables $Y_i$. For $i \in I$, $Y_i$ must be equal to ${ B_i x x^T B_i }/{x^T B_i x} $, while if $i \in I^c$, we must have $Y_i \succeq B_i  -  {   B_i xx^T B_i   }/{x^T B_i x}$, which is a stricter condition than $Y_i \succeq B_i $ (because $x^T B_i x < 0$).

\subsection{Efficient Optimality Conditions }
\label{sec:eff-opt} The condition presented in Proposition~\ref{prop:opt-cond} still requires solving a large semidefinite program. In practice, good candidates for $Y_i,~i\in I^c$ can be found by solving for minimum trace matrices satisfying the feasibility conditions of proposition \ref{prop:opt-cond}. As we will see below, this can be formulated as a semidefinite program which can be solved explicitly.
\begin{lemma} \label{lem:eff-opt}
Let $A\in\reals^{n \times n}$, $\rho \geq 0$, $x\in\reals^n$ and $B_i=a_ia_i^T - \rho\idm$ with $a_1,\dots,a_n \in \reals^n$ the columns of $A$. If $(a_i^T x)^2 < \rho$ and $\|x\|=1$, an optimal solution of the semidefinite program:
\[\BA{ll}
\mbox{minimize} & \Tr Y_i\\
\mbox{subject to} &Y_i\succeq B_i  -  \frac{   B_i xx^T B_i   }{x^T B_i x} ,~ x^TY_ix=0,~Y_i\succeq 0,
\EA\]
is given by:
\BEQ\label{eq:yic}
Y_i=\max\left\{0,\rho\frac{(a_i^T a_i-\rho)}{(\rho-(a_i^T x)^2)}\right\}\frac{(\idm-xx^T)a_ia_i^T(\idm-xx^T)}{\|(\idm-xx^T)a_i\|^2}.
\EEQ
\end{lemma}
\begin{proof} Let us write $M_i=B_i - \frac{B_i xx^T B_i}{x^T B_i x}$, we first compute:
\BEAS
a_i^T M_i a_i
& = & ( a_i^T a_i - \rho) a_i^T a_i - \frac{ ( a_i^Ta_i a_i^T x - \rho a_i^Tx)^2}{(a_i^Tx)^2-\rho} \\
& = &  \frac{  ( a_i^T a_i - \rho) }{\rho-(a_i^T x)^2} \rho ( a_i^T a_i -  (a_i^Tx)^2 ).
\EEAS
When $a_i^T a_i \leq \rho$, the matrix $M_i$ is negative semidefinite, because $\|x\|=1$ means $a_i^TMa_i \leq0$ and $x^TMx=a_i^TMx=0$. The solution of the minimum trace problem is then simply $Y_i=0$. We now assume that $a_i^T a_i > \rho$ and first check feasibility of the candidate solution $Y_i$ in (\ref{eq:yic}). 
By construction, we have $Y_i\succeq 0$ and $Y_i x = 0$, and a short calculation shows that:
\BEAS
a_i^T Y_i a_i  & = & \rho\frac{(a_i^T a_i-\rho)}{(\rho-(a_i^T x)^2)} ( a_i^T a_i - ( a_i^Tx)^2 )\\
& = &  a_i^T M_i a_i.
\EEAS
We only need to check that $Y_i\succeq M_i$ on the subspace spanned by $a_i$ and $x$, for which there is equality. This means that $Y_i$ in (\ref{eq:yic}) is feasible and we now check its optimality. The dual of the original semidefinite program can be written:
\[\BA{ll}
\mbox{maximize} & \Tr P_i M_i \\
\mbox{subject to} & \idm - P_i + \nu xx^T \succeq 0\\
& P_i \succeq 0,
\EA\]
and the KKT optimality conditions for this problem are written:
\[ \left\{\BA{l}
Y_i(\idm-P_i+\nu xx^T)=0,~P_i(Y_i- M_i )=0,\\
\idm-P_i+\nu xx^T \succeq 0,\\ 
P_i \succeq 0,~ Y_i \succeq 0,~Y_i\succeq M_i, ~ Y_i xx^T=0,\quad i\in I^c.
\EA\right.\]
Setting $P_i = Y_i { \Tr Y_i }/{\Tr Y_i^2}$ and $\nu$ sufficiently large makes these variables dual feasible. Because all contributions of $x$ are zero, $\Tr Y_i (Y_i- M_i) $ is proportional to $\Tr a_i a_i^T (Y_i- M_i)$ which is equal to zero and $Y_i$ in (\ref{eq:yic}) satisifies the KKT optimality conditions.
\end{proof}

We summarize the results of this section in the theorem below, which provides sufficient optimality conditions on a sparsity pattern $I$.

\begin{theorem} \label{th:opt}
Let $A\in\reals^{n \times n}$, $\rho \geq 0$, $\Sigma=A^TA$ with $a_1,\dots,a_n \in \reals^n$ the columns of $A$. Given a sparsity pattern $I$, setting $x$ to be the largest eigenvector of $\sum_{i \in I} a_i a_i^T$, if there is a $\rho^*\geq 0$ such that the following conditions hold:
\[
\max_{i \in I^c} (a_i^T x)^2 < \rho^* < \min_{i \in I} (a_i^T x)^2 \quad \mbox{and}\quad
\lambdamax\left(\sum_{i=1}^n Y_i\right) \leq \sum_{i \in I}((a_i^T x)^2 - \rho^*),
\] 
with the dual variables $Y_i$ for $i\in I^c$ defined as in (\ref{eq:yic}) and:
\[
Y_i = \frac{ B_i x x^T B_i }{x^T B_i x},\quad{when }~ i \in I,
\]
then the sparsity pattern $I$ is globally optimal for the sparse PCA problem~(\ref{eq:pca-card}) with $\rho=\rho^*$ and we can form an optimal solution $z$ by solving the maximum eigenvalue problem:
\[
z=\argmax_{ \{z_{I^c}=0,~\|z\|=1 \}} z^T \Sigma z.\\
\]
\end{theorem}
\begin{proof}
Following proposition \ref{prop:opt-cond} and lemma \ref{lem:eff-opt}, the matrices $Y_i$  are dual optimal solutions corresponding to the primal optimal solution $X=xx^T$ in (\ref{eq:pca-relax}). Because the primal solution has rank one, the semidefinite relaxation (\ref{eq:primal-p}) is tight so the pattern $I$ is optimal for~(\ref{eq:pca-card}) and \mysec{spca} shows that $z$ is a globally optimal solution to~(\ref{eq:pca-card}) with $\rho=\rho^*$.
\end{proof}

\subsection{Gap minimization: finding the optimal $\rho$}
All we need now is an efficient algorithm to find $\rho^*$ in theorem \ref{th:opt}. As we will show below, when the dual variables $Y_i^c$ are defined as in (\ref{eq:yic}), the duality gap in (\ref{eq:pca-card}) is a convex function of $\rho$ hence, given a sparsity pattern $I$, we can efficiently search for the best possible $\rho$ (which must belong to an \emph{interval}) by performing a few binary search iterations. 

\begin{lemma} \label{lemma:cvx-gap}
Let $A\in\reals^{n \times n}$, $\rho \geq 0$, $\Sigma=A^TA$ with $a_1,\dots,a_n \in \reals^n$ the columns of $A$. Given a sparsity pattern $I$, setting $x$ to be the largest eigenvector of $\sum_{i \in I} a_i a_i^T$, with the dual variables $Y_i$ for $i\in I^c$ defined as in (\ref{eq:yic}) and:
\[
Y_i = \frac{ B_i x x^T B_i }{x^T B_i x},\quad{when }~ i \in I.
\]
The duality gap in (\ref{eq:pca-card}) which is given by:
\[
\mathrm{gap}(\rho) \equiv \lambdamax\left(\sum_{i=1}^n Y_i\right)-\sum_{i \in I}((a_i^T x)^2 - \rho),
\]
is a convex function of $\rho$ when
\[
\max_{i \notin I} (a_i^T x)^2 < \rho < \min_{i \in I} (a_i^T x)^2.
\]
\end{lemma}
\begin{proof} For $i \in I$ and $u\in\reals^n$, we have
\[
u^T Y_i u = \frac{ ( u^T a_i a_i^T x - \rho u^Tx )^2}{ (a_i^T x)^2 - \rho},
\] 
which is a convex function of $\rho$ \citep[p.73]{Boyd03}. For $i \in I^c$, we can write:
\[
\frac{ \rho ( a_i^T a_i - \rho)}{ \rho - (a_i^T x)^2} = - \rho + ( a_i^T a_i - ( a_i^T x)^2) \left( 1  +\frac{ (a_i^T x)^2}{\rho - (a_i^T x)^2} \right),
\] 
hence $\max \{ 0, {\rho ( a_i^T a_i - \rho)}/{ (\rho - (a_i^T x)^2}) \}$ is also a convex function of $\rho$. This means that:
\[
u^TY_iu=\max\left\{0,\rho\frac{(a_i^T a_i-\rho)}{(\rho-(a_i^T x)^2)}\right\}\frac{(u^Ta_i -(x^Tu)(x^Ta_i))^2}{\|(\idm-xx^T)a_i\|^2}
\]
is convex in $\rho$ when $i \in I^c$. We conclude that $\sum_{i=1}^n u^TY_i u$ is convex, hence:
\[
\mathrm{gap}(\rho)=\max_{\|u\|=1} ~\sum_{i=1}^n u^TY_i u - \sum_{i \in I}((a_i^T x)^2 - \rho)
\]
is also convex in $\rho$ as a pointwise maximum of convex functions of $\rho$.
\end{proof}

This result shows that the set of $\rho$ for which the pattern $I$ is optimal must be an interval. It also suggests an efficient procedure for testing the optimality of a given pattern $I$. We first compute $x$ as a leading eigenvector $\sum_{i \in I} a_i a_i^T$. We then compute an interval in $\rho$ for which $x$ satisfies the basic consistency condition:
\[
\max_{i \notin I} (a_i^T x)^2 \equiv \rho_\mathrm{min} \leq \rho \leq \rho_\mathrm{max} \equiv \min_{i \in I} (a_i^T x)^2.
\]
Note that this interval could be empty, in which case $I$ cannot be optimal. We then minimize $\mathrm{gap}(\rho)$ over the interval $[\rho_\mathrm{min},\rho_\mathrm{max}]$. If the minimum is zero for some $\rho=\rho^*$, then the pattern $I$ is optimal for the sparse PCA problem in (\ref{eq:pca-card}) with $\rho=\rho^*$. 

Minimizing the convex function $\mathrm{gap}(\rho)$ can be done very efficiently using binary search. The initial cost of forming the matrix $\sum_{i=1}^nY_i$, which is a simple outer matrix product, is $O(n^3)$. At each iteration of the binary search, a subgradient of $\mathrm{gap}(\rho)$ can then be computed by solving a maximum eigenvalue problem, at a cost of $O(n^2)$. This means that the complexity of finding the optimal $\rho$ over a given interval $[\rho_\mathrm{min},\rho_\mathrm{max}]$ is $O(n^2 \log_2 ((\rho_\mathrm{max}-\rho_\mathrm{min})/\epsilon))$, where $\epsilon$ is the target precision. Overall then, the total cost of testing the optimality of a pattern $I$ is $O(n^3 + n^2 \log_2 ((\rho_\mathrm{max}-\rho_\mathrm{min})/\epsilon))$.

Note that an additional benefit of deriving explicit dual feasible points $Y_i$ is that plugging these solutions into the objective of problem (\ref{eq:dual-y}):
\[\BA{ll}
\mbox{min.} & \lambdamax\left(\sum_{i=1}^n Y_i\right)\\
\mbox{s.t.} & Y_i\succeq B_i,~Y_i\succeq0,\quad i=1,\ldots,n.
\EA\]
produces an \emph{upper bound} on the optimum value of the original sparse PCA problem (\ref{eq:pca-card}) even when the pattern $I$ is not optimal (all we need is a $\rho$ satisfying the consistency condition).

\subsection{Solution improvements and randomization}
When these conditions are not satisfied, the relaxation~(\ref{eq:primal-p}) has an optimal solution with rank strictly larger than one, hence is not tight. At such a point, we can use a different relaxation such as DSPCA by \citet{dasp04a} to try to get a better solution. We can also apply randomization techniques to improve the quality of the solution of problem (\ref{eq:primal-p})~\citep{Ben-02}.

\section{Applications}
\label{sec:apps} In this section, we discuss some applications of sparse PCA to subset selection and compressed sensing.

\subsection{Subset selection}
\label{sec:subset} We consider $p$ data points in $\reals^n$, in a data matrix $X \in \reals^{p \times n}$. We assume that we are given real numbers $y \in \reals^p$ to predict from $X$ using linear regression, estimated by least squares. We are thus looking for $w \in \reals^n$ such that $\|y-Xw\|^2$ is minimum. In the subset selection problem, we are looking for sparse coefficients $w$, i.e., a vector $w$ with many zeros. We thus consider the problem:
\BEQ
s(k) = \min_{w \in \reals^n, \ \Card{w} \leq k } \| y- Xw \|^2 .
\EEQ
Using the sparsity pattern $u \in \{0,1\}^n$, and optimizing with respect to $w$, we have
\BEQ
s(\rho) = \min_{u \in \{0,1\}^n , \ \ones^T  u \leq k }  \|y\|^2 - y^T  X(u) (    X(u)^T  X(u)  )^{-1} X(u)^T  y, 
\EEQ
where $X(u) = X \diag(u)$. We can rewrite
$y^T  X(u) (    X(u)^T  X(u)  )^{-1} X(u)^T  y$ as the largest generalized eigenvalue of the pair $( X(u)^T  yy^T  X(u), X(u)^T  X(u))$, i.e., as
$$
y^T  X(u) (    X(u)^T  X(u)  )^{-1} X(u)^T  y = 
\max_{w \in \reals^n} \frac{ w^T  X(u)^T  yy^T  X(u) w}{
w^T  X(u)^T  X(u) w}.
$$
We thus have:
\BEQ
s(k) = \|y\|^2 - \max_{u \in \{0,1\}^n , \ones^T  u \leq k } \max_{w \in \reals^n }   \frac{ w^T  \diag(u) X^T  yy^T  X \diag(u)  w}{
w^T  \diag(u)  X^T  X \diag(u) ) w} .
\EEQ
Given a pattern $u \in \{0,1\}^n$, let 
\[
s_0 = y^T  X(u) (    X(u)^T  X(u)  )^{-1} X(u)^T  y
\]
be the largest generalized eigenvalue corresponding to the pattern $u$. The pattern is optimal if and only if the largest generalized eigenvalue of the pair $\{ X(v)^T  yy^T  X(v), X(v)^T  X(v)\}$ is less than $s_0$ for any $v \in \{0,1\}^n$ such that  $v^T \ones = u^T \ones$.
This is equivalent to the optimality of $u$ for the sparse PCA problem with matrix  $X^T  yy^T  X - s_0 X^T  X$, which can be checked using the sparse PCA optimality conditions derived in the previous sections. 

Note that unlike in the sparse PCA case, this convex relaxation does not immediately give a simple bound on the optimal value of the subset selection problem. However, we get a bound of the following form: when $v \in \{0,1\}^n$ and $w \in \reals^n$ is such that $\ones^T v = k$ with: 
\[ 
w^T \left( X(v)^T  yy^T  X(v) - s_0 X(v)^T  X(v) \right) w \leq B,
\] 
where $B\geq 0$ (because $s_0$ is defined from $u$), we have:
\BEAS
\| y\|^2 - s_0 \geq  s(k) & \geq  & \| y\|^2 - s_0 - B \left(  \min_{v 
\in \{0,1\}^n, \ones^T v = k} \lambda_{\min}
(   X(v)^T  X(v)  ) \right)^{-1}
 \\
 & \geq & \| y\|^2 - s_0 - B \left(   \lambda_{\min}
(   X^T  X  ) \right)^{-1}.
 \EEAS
This bound gives a sufficient condition for optimality in subset selection, for any problem instance and any given subset. This is to be contrasted with the sufficient conditions derived for particular algorithms, such as the LASSO~\citep{yuanlin,Zhaoyu} or backward greedy selection~\citep{couvreur}. Note that some of these optimality conditions are often based on sparse eigenvalue problems (see \citet[\S 2]{Mein06a}), hence our convex relaxations helps both in checking sufficient conditions for optimality (before the algorithm is run) and in testing a posteriori the optimality of a particular solution.
 
\subsection{Sparse recovery}\label{sec:recov}
Following \cite{Cand05} (see also \cite{Dono05}), we seek to recover a signal $f\in\reals^n$ from corrupted measurements $y=Af+e$, where $A\in\reals^{m \times n}$ is a coding matrix and $e\in\reals^m$ is an unknown vector of errors with low cardinality. This can be reformulated as the problem of finding the sparsest solution to an underdetermined linear system:
\BEQ \label{eq:lp-mincard}
\BA{ll}
\mbox{minimize} & \|x\|_0\\
\mbox{subject to} & Fx=Fy\\
\EA\EEQ
where $\|x\|_0=\Card(x)$ and $F\in\reals^{p \times m}$ is a matrix such that $FA=0$. A classic trick to get good approximate solutions to problem (\ref{eq:lp-mincard}) is to substitute the (convex) $\ell_1$ norm to the (combinatorial) $\ell_0$ objective above, and solve instead:
\BEQ \label{eq:lp-minone}
\BA{ll}
\mbox{minimize} & \|x\|_1\\
\mbox{subject to} & Fx=Fy,
\EA\EEQ
which is equivalent to a linear program in $x\in\reals^m$. Following \cite{Cand05}, given a matrix $F \in\reals^{p \times m}$  and an integer $S$ such that $0<S\leq m$,  we define its \emph{restricted isometry} constant $\delta_S$ as the smallest number such that for any subset $I\subset [1,m]$ of cardinality at most $S$ we have:
\BEQ\label{eq:ricd}
(1-\delta_S)\|c\|^2 \leq \|F_Ic\|^2 \leq (1+\delta_S)\|c\|^2,
\EEQ
for all $c\in\reals^{|I|}$, where $F_I$ is the submatrix of $F$ formed by keeping only the columns of $F$ in the set $I$. The following result then holds.
\begin{proposition} \textbf{(\cite{Cand05})}. 
Suppose that the restricted isometry constants of a matrix $F \in\reals^{p \times m}$ satisfy 
\BEQ\label{eq:ric}
\delta_S+\delta_{2S}+\delta_{3S} < 1
\EEQ
for some integer $S$ such that $0<S\leq m$, then if $x$ is an optimal solution of the convex program:
\[
\BA{ll}
\mbox{minimize} & \|x\|_1\\
\mbox{subject to} & Fx=Fy\\
\EA
\]
such that $\Card{x}\leq S$ then $x$ is also an optimal solution of the combinatorial problem:
\[
\BA{ll}
\mbox{minimize} & \|x\|_0\\
\mbox{subject to} & Fx=Fy.\\
\EA\
\]
\end{proposition}
In other words, if condition (\ref{eq:ric}) holds for some matrix $F$ such that $FA=0$, then perfect recovery of the signal $f$ given $y=Af+e$ provided the error vector satisfies $\Card(e)\leq S$. Our key observation here is that the restricted isometry constant $\delta_S$ in condition (\ref{eq:ric}) can be computed by solving the following sparse maximum eigenvalue problem:
\[\BA{rll}
(1+\delta_S)\leq&\mbox{max.} & x^T(F^TF)x\\
&\mbox{s. t.} & \Card(x)\leq S\\
&& \|x\|=1,\\
\EA\]
in the variable $x\in\reals^m$ and another sparse maximum eigenvalue problem on $\alpha \idm-FF^T$ with $\alpha$ sufficiently large, with $\delta_S$ computed from the tightest one. In fact, (\ref{eq:ricd}) means that:
\BEAS
(1+\delta_S) &\leq& \max_{\{I\subset[1,m]:~|I|\leq S\}}~\max_{\|c\|=1} c^TF_I^TF_Ic\\
&=& \max_{\{u\in\{0,1\}^n:~\ones^Tu\leq S\}}~\max_{\|x\|=1} x^T\diag(u)F^TF\diag(u)x\\
&=& \max_{\{\|x\|= 1,~\Card(x)\leq S\}} x^TF^TFx,
\EEAS
hence we can compute an upper bound on $\delta_S$ by duality, with:
\[
(1+\delta_S) \leq \inf_{\rho \geq 0} \phi(\rho) + \rho S
\]
where $\phi(\rho)$ is defined in (\ref{eq:pca-card}). This means that while \cite{Cand05} obtained an asymptotic proof that some random matrices satisfied the restricted isometry condition (\ref{eq:ric}) with overwhelming probability (i.e. exponentially small probability of failure), whenever they are satisfied, the \emph{tractable} optimality conditions and upper bounds we obtain in \mysec{tightness} for sparse PCA problems allow us to prove, \emph{deterministically}, that a finite dimensional matrix satisfies the restricted isometry condition in (\ref{eq:ric}). Note that \cite{Cand05} provide a slightly weaker condition than (\ref{eq:ric}) based on restricted orthogonality conditions and extending the results on sparse PCA to these conditions would increase the maximum $S$ for which perfect recovery holds. In practice however, we will see in Section \ref{ss:num-sparse-rec} that the relaxations in (\ref{eq:dual-y}) and \cite{dasp04a} do provide very tight upper bounds on sparse eigenvalues of random matrices but solving these semidefinite programs for very large scale instances remains a significant challenge.

\section{Numerical Results}
\label{sec:numer}
In this section, we first compare the various methods detailed here on artificial examples, then test their performance on a biological data set. PathSPCA, a MATLAB code reproducing these results may be downloaded from the authors' web pages.

\subsection{Artificial Data}
We generate a matrix $U$ of size 150 with uniformly distributed
coefficients in $[0,1]$. We let $v\in\reals^{150}$ be a sparse vector with:
\[
v_i=\left\{\BA{ll}
1 & \mbox{if } i \leq 50\\
1/(i-50) & \mbox{if } 50 < i \leq 100\\
0 & \mbox{otherwise}\\ \EA \right.
\]
We form a test matrix $\Sigma=U^TU+\sigma vv^T$, where $\sigma$ is the signal-to-noise ratio. We first compare the relative performance of the algorithms in Section \ref{sec:greedy} at identifying the correct sparsity pattern in $v$ given the matrix $\Sigma$. The resulting ROC curves are plotted in figure \ref{fig:roc} for $\sigma=2$. On this example, the computing time for the approximate greedy algorithm in \mysec{greedy-path} was 3 seconds versus 37 seconds for the full greedy solution in \mysec{greedy-classic}. Both algorithms produce almost identical answers. We can also see that both sorting and thresholding ROC curves are dominated by the greedy algorithms.
\begin{figure}[ht]
\begin{center}
\psfrag{fpr}[t][b]{False Positive Rate}
\psfrag{tpr}[b][t]{True Positive Rate}
\includegraphics[width=.6\textwidth]{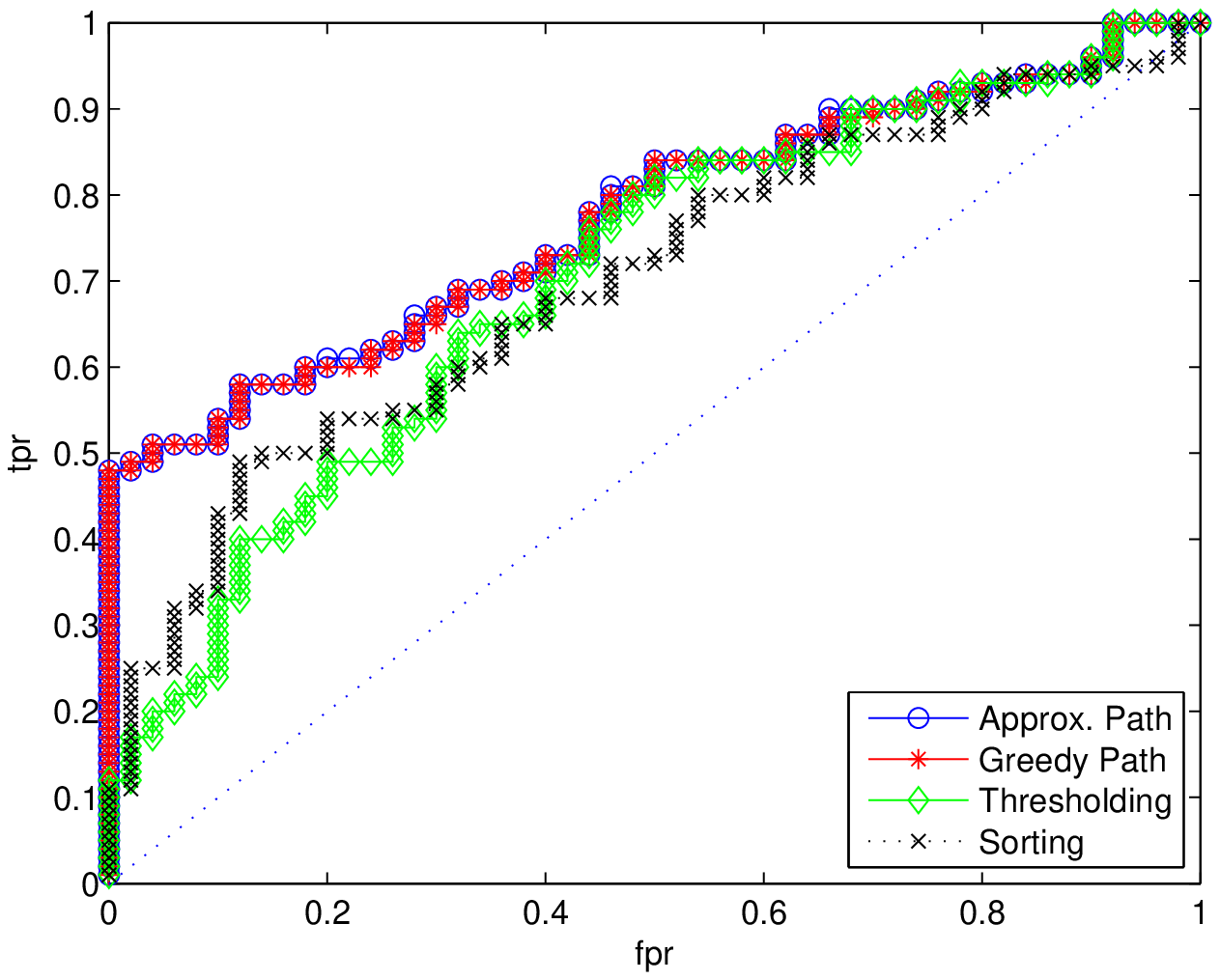} \caption{ROC curves for sorting, thresholding, fully greedy solutions (\mysec{greedy-classic}) and approximate greedy solutions (\mysec{greedy-path}) for $\sigma=2$. \label{fig:roc}}
\end{center}
\end{figure}

We then plot the variance versus cardinality tradeoff curves for various values of the signal-to-noise ratio. In figure \ref{fig:dspca},  We notice that the magnitude of the error (duality gap) decreases with the signal-to-noise ratio. Also, because of the structure of our problem, there is a kink in the variance at the (exact) cardinality 50 in each of these curves. Note that for each of these examples, the error (duality gap) is minimal precisely at the kink.

Next, we use the DSPCA algorithm of \citet{dasp04a} to find better solutions where the greedy codes have failed to obtain globally optimal solutions. In \citet{dasp04a}, it was shown that an upper bound on (\ref{eq:pca-card}) can be computed as:
\[
\phi(\rho) \leq \min_{|U_{ij}|\leq \rho} \lambdamax(\Sigma+U).
\]
which is a convex problem in the matrix $U\in\symm_n$. Note however that the cost of solving this relaxation for a \emph{single} $\rho$ is $O(n^{4}\sqrt{\log n})$ versus $O(n^3)$ for a full path of approximate solutions. Also, the results in \citet{dasp04a} do not provide any hint on the value of $\rho$, but we can use the breakpoints coming from suboptimal points in the greedy search algorithms in \mysec{greedy-path} and the consistency intervals in \eq{rho-interval}. In figure \ref{fig:dspca} we plot the variance versus cardinality tradeoff curve for $\sigma=10$. We plot greedy variances (solid line), dual upper bounds from \mysec{eff-opt} (dotted line) and upper bounds computed using DSPCA (dashed line). 

\begin{figure}[ht]
\begin{center}
\psfrag{card}[t][b]{Cardinality}
\psfrag{var}[b][t]{Variance}
\includegraphics[width=.49\textwidth]{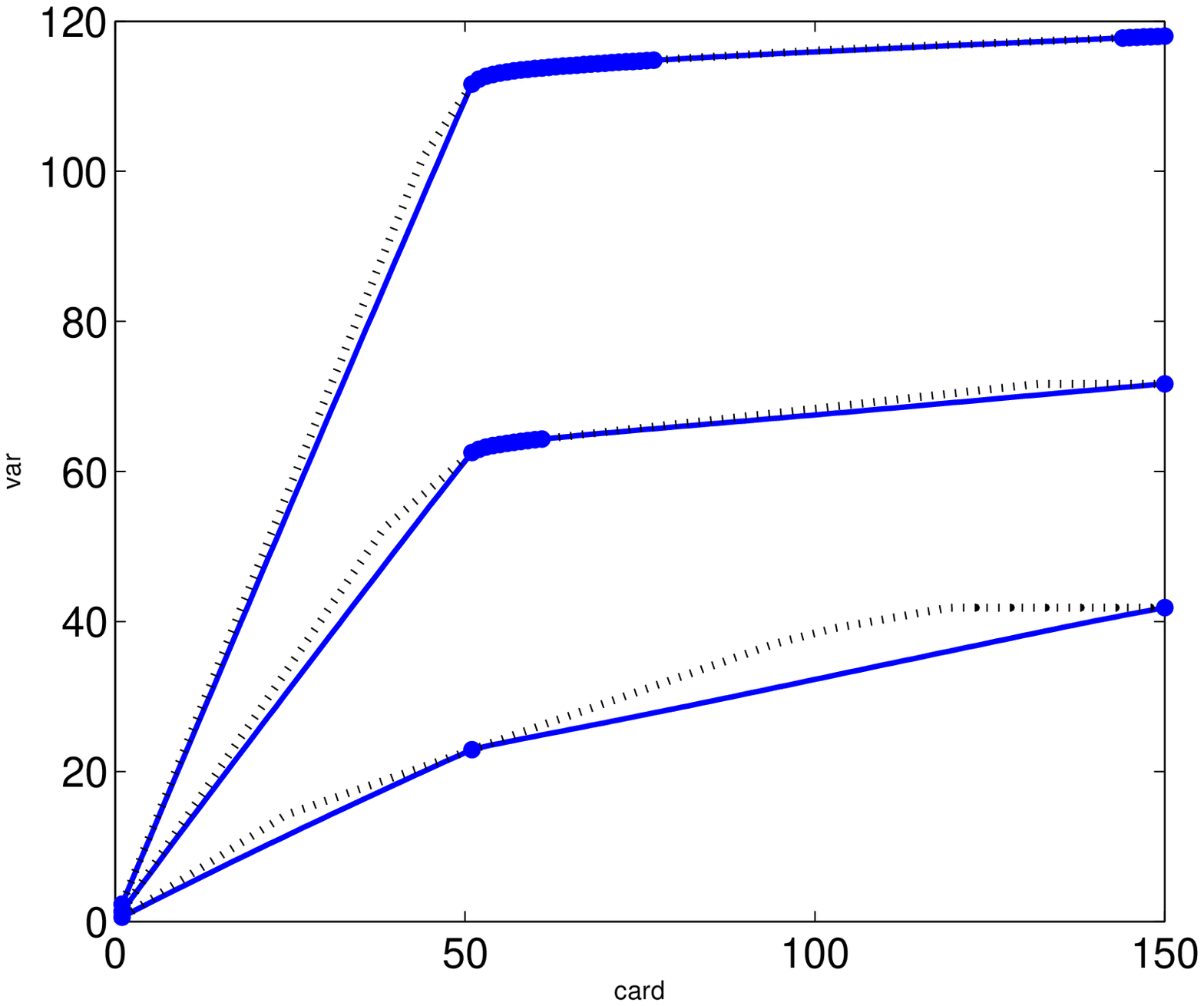}
\includegraphics[width=.49\textwidth]{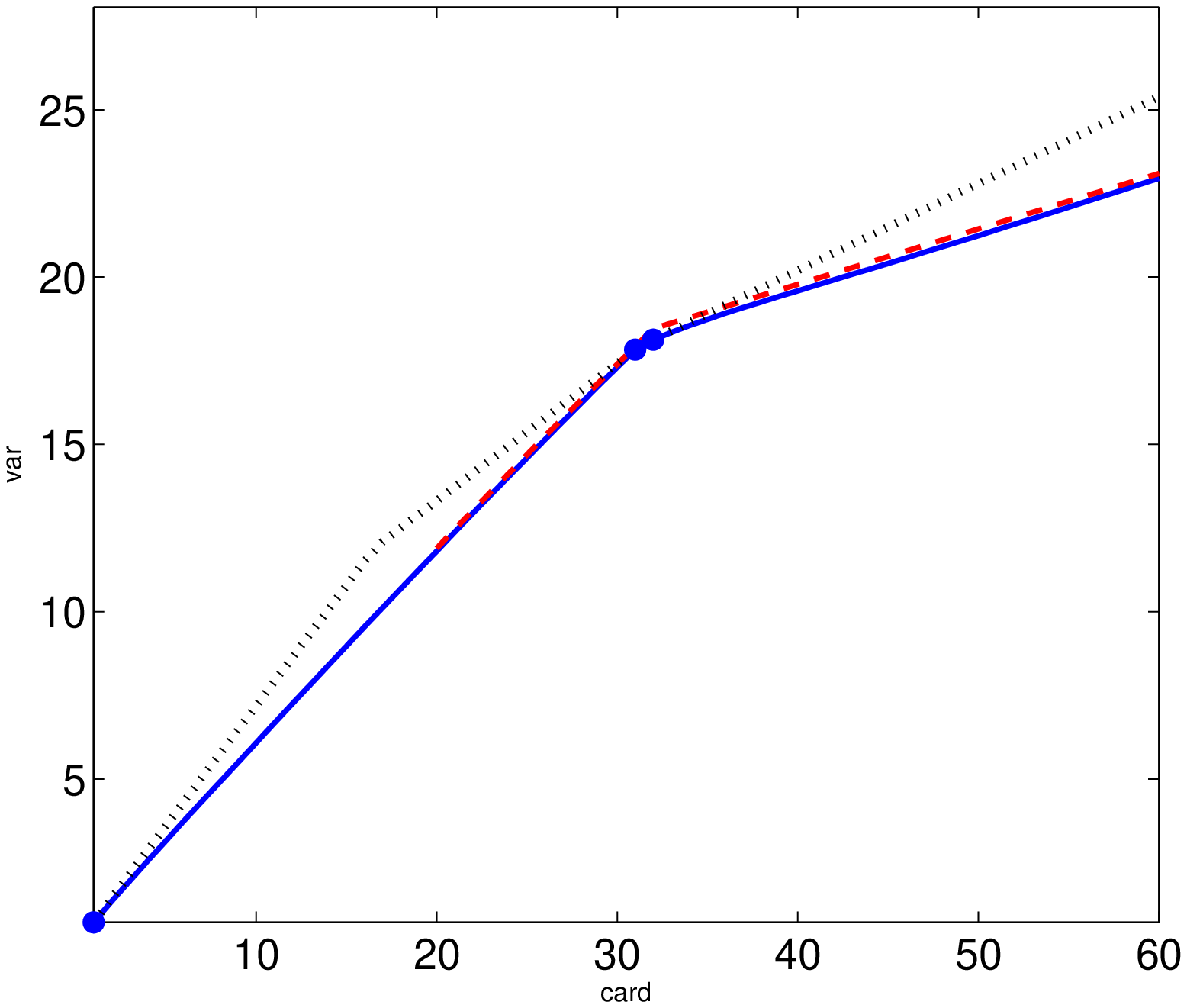} 
\caption{\emph{Left:} variance versus cardinality tradeoff curves for $\sigma=10$ (bottom), $\sigma=50$ and $\sigma=100$ (top). We plot the variance (solid line) and the dual upper bounds from \mysec{eff-opt} (dotted line) for each target cardinality. \emph{Right:} variance versus cardinality tradeoff curve for $\sigma=10$. We plot greedy variances (solid line), dual upper bounds from \mysec{eff-opt} (dotted line) and upper bounds computed using DSPCA (dashed line). Optimal points (for which the relative duality gap is less than $10^{-4}$) are in bold.\label{fig:dspca}}
\end{center}
\end{figure}

\subsection{Subset selection}
We now present simulation experiments on synthetic datasets for the subset selection problem. We consider data sets generated from a sparse linear regression problem and study optimality for the subset selection problem, given the exact cardinality of the generating vector. In this setting, it is known that regularization by the $\ell_1$-norm, a procedure also known as the Lasso~\citep{tibs96}, will asymptotically lead to the correct solution if and only if a certain consistency condition is satisfied~\citep{yuanlin,Zhaoyu}. Our results provide here a tractable test the optimality of solutions obtained from various algorithms such as the Lasso, forward greedy or backward greedy algorithms.
 
In \myfig{optimality}, we consider two pairs of randomly generated examples in dimension 16, one for which the lasso is provably consistent, one for which it isn't. We perform 50 simulations with 1000 samples and varying noise and compute the average frequency of optimal subset selection for Lasso and greedy backward algorithm together with the frequency of provable optimality (i.e., where our method did ensure a posteriori that the point was optimal). We can see that the backward greedy algorithm exhibits good performance (even in the Lasso-inconsistent case) and that our sufficient optimality condition is satisfied as long as there is not too much noise. In \myfig{MSE}, we plot the average mean squared error versus cardinality, over 100 replications, using forward (dotted line) and backward (circles) selection, the Lasso (large dots) and exhaustive search (solid line). The plot on the left shows the results when the Lasso consistency condition is satisfied, while the plot on the right shows the mean squared errors when the consistency condition is not satisfied.
The two sets of figures do show that the LASSO is consistent only when the consistency condition is satisfied, while the backward greedy algorithm finds the correct pattern if the noise is small enough~\citep{couvreur} even in the LASSO inconsistent case.

\begin{figure}[hp]
\begin{center}
\psfrag{logsig}[t][b]{Noise Intensity}
\psfrag{ProbOpt}[b][t]{Probability of Optimality}
\includegraphics[scale=.41]{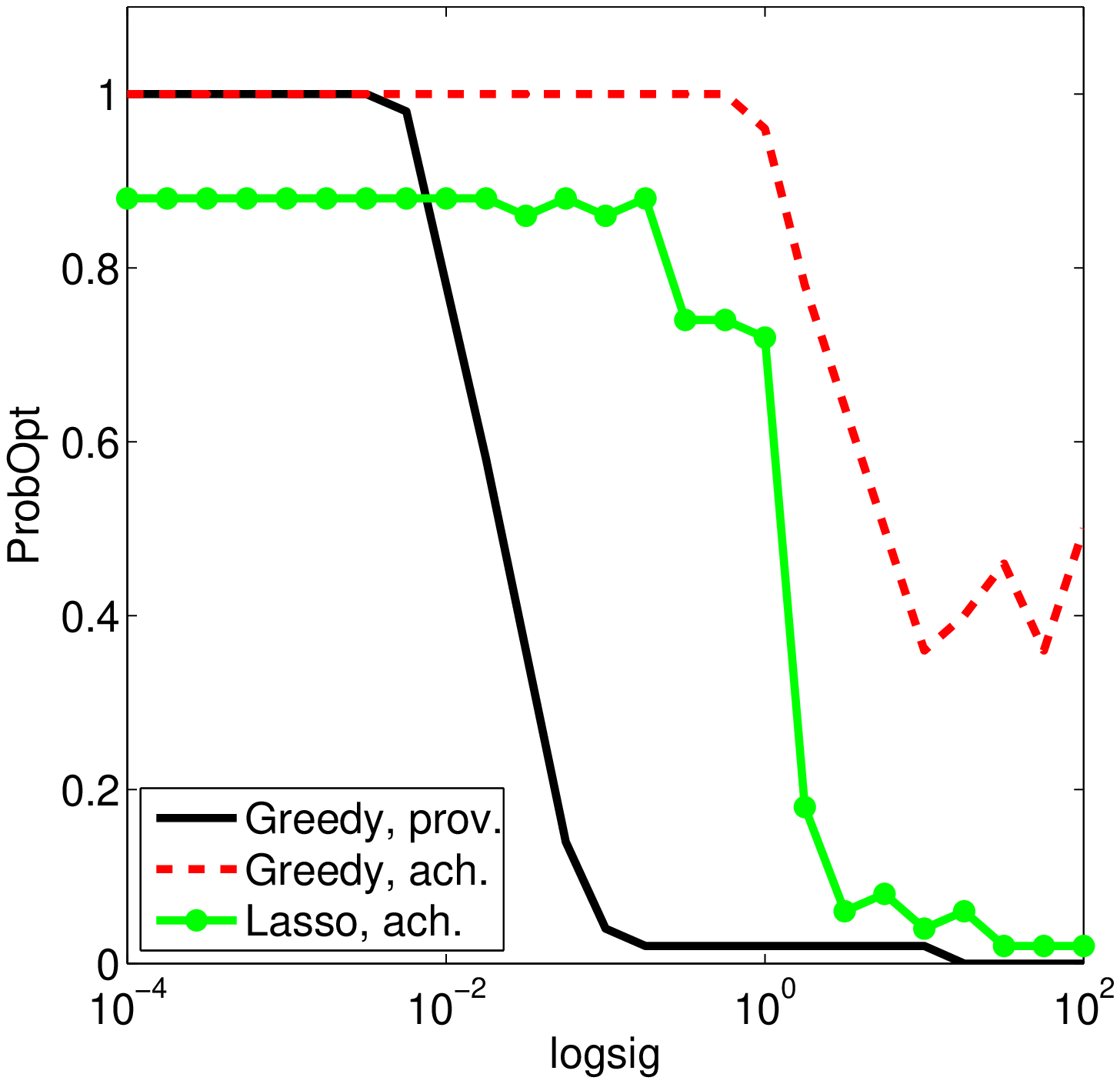}
\includegraphics[scale=.41]{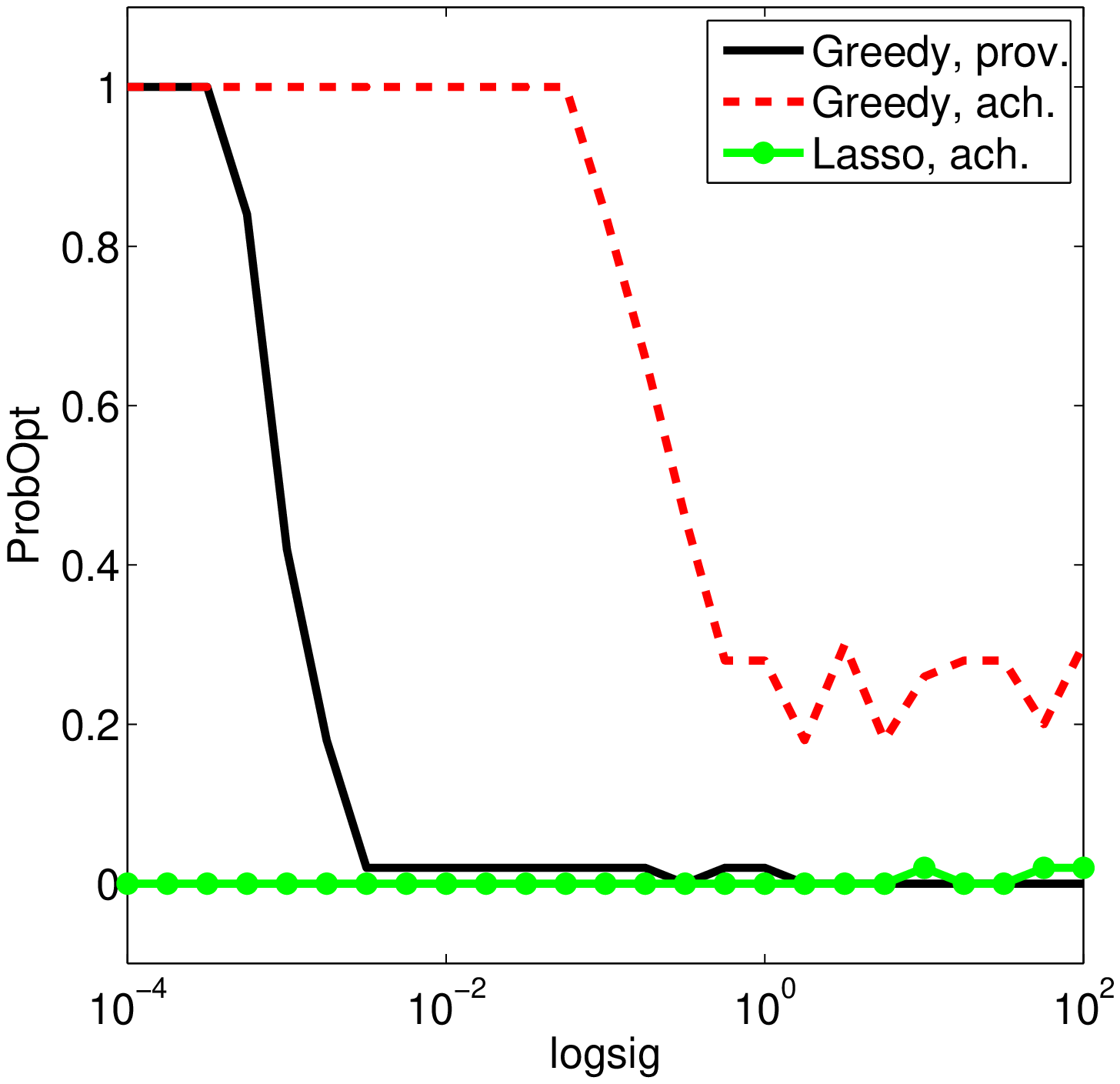}
\caption{Backward greedy algorithm and Lasso. We plot the probability of achieved (dotted line) and provable (solid line) optimality versus noise for greedy selection against Lasso (large dots), for the subset selection problem on a noisy sparse vector. \emph{Left:} Lasso consistency condition satisfied. \emph{Right:} consistency condition not satisfied.\label{fig:optimality}}
\end{center}
\end{figure}
 
\begin{figure}[hp]
\begin{center}
\psfrag{subcard}[t][b]{Subset Cardinality}
\psfrag{mse}[b][t]{Mean Squared Error}
\includegraphics[scale=.39]{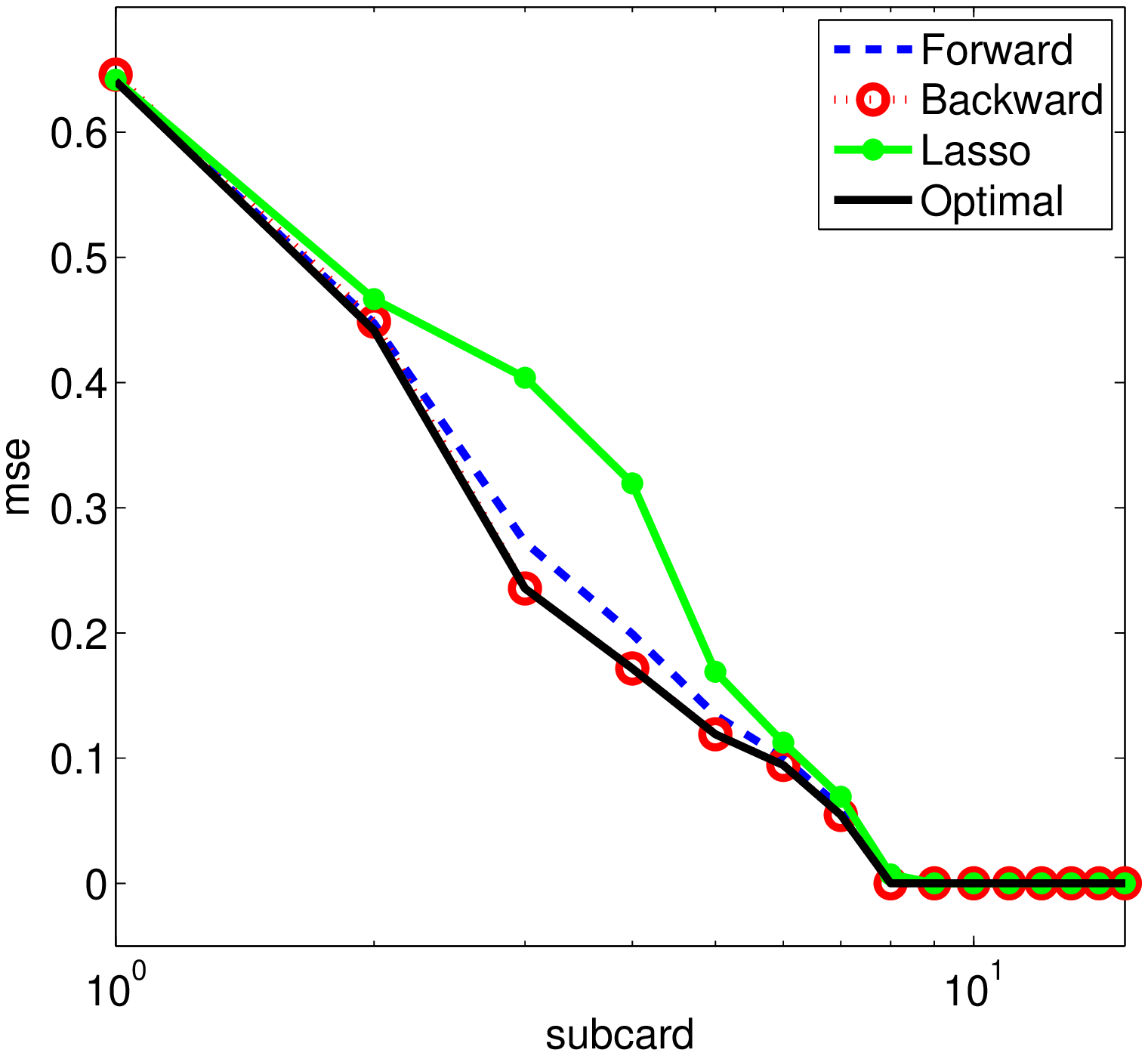}
\includegraphics[scale=.39]{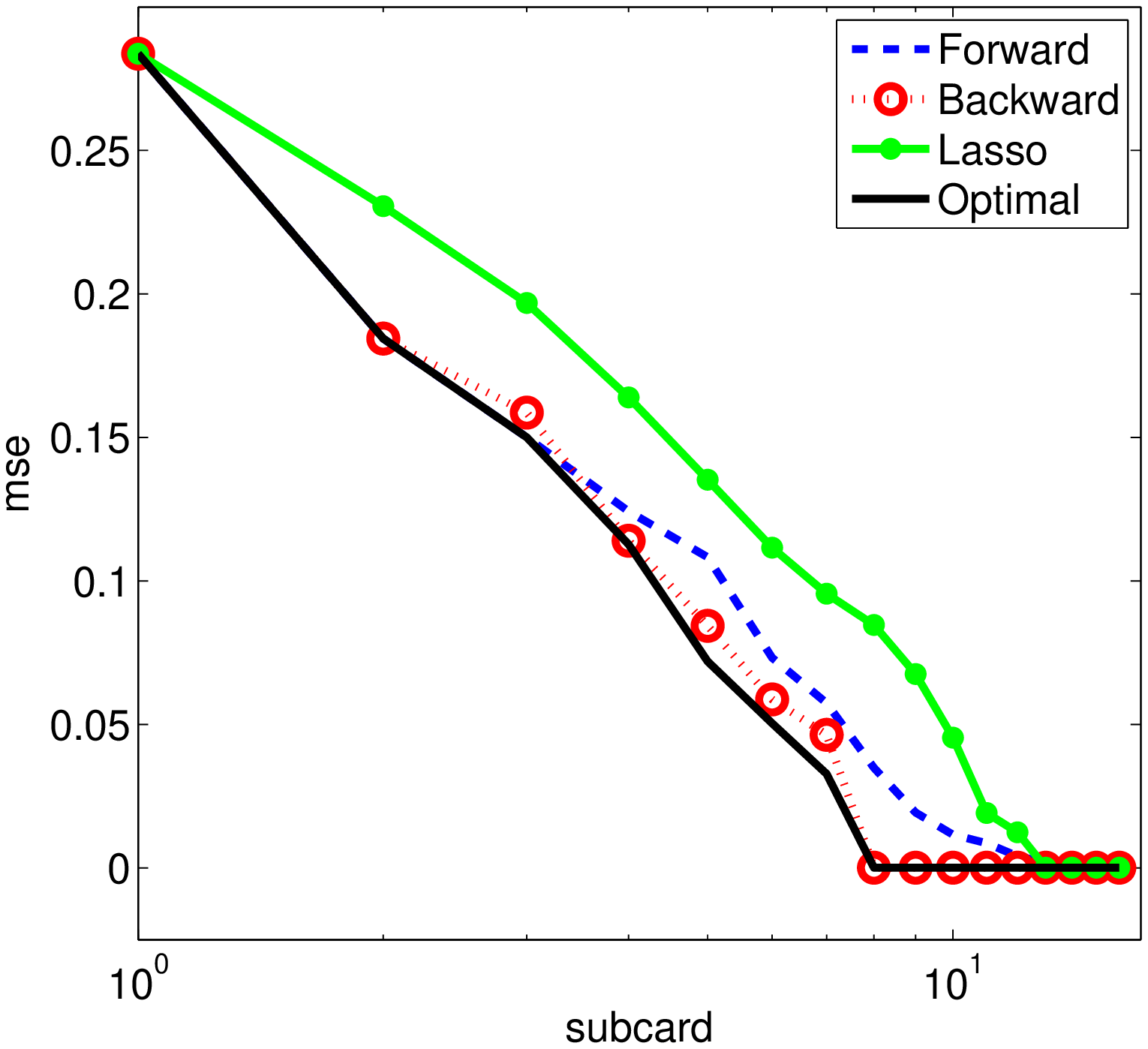}
\caption{Greedy algorithm and Lasso. We plot the average mean squared error versus cardinality, over 100 replications, using forward (dotted line) and backward (circles) selection, the Lasso (large dots) and exhaustive search (solid line). \emph{Left:} Lasso consistency condition satisfied. \emph{Right:} consistency condition not satisfied.\label{fig:MSE}}
\end{center}
\end{figure}

\subsection{Sparse recovery}\label{ss:num-sparse-rec}
Following the results of \mysec{recov}, we compute the upper and lower bounds on sparse eigenvalues produced using various algorithms. We study the following problem:
\[\BA{rll}
\mbox{maximize} & x^T\Sigma x\\
\mbox{subject to} & \Card(x)\leq S\\
& \|x\|=1,\\
\EA\]
where we pick $F$ to be normally distributed and small enough so that computing sparse eigenvalues by exhaustive search is numerically feasible. We plot the maximum sparse eigenvalue versus cardinality, obtained using exhaustive search (solid line), the approximate greedy (dotted line) and fully greedy (dashed line) algorithms. We also plot the upper bounds obtained by minimizing the gap of a rank one solution (squares), by solving the semidefinite relaxation explicitly (stars) and by solving the DSPCA dual (diamonds). On the left, we use a matrix $\Sigma=F^TF$ with $F$ Gaussian. On the right, $\Sigma=uu^T/\|u\|^2+2V$, where $u_i=1/i,~i=1,\ldots,n$ and $V$ is matrix with coefficients uniformly distributed in $[0,1]$. Almost all algorithms are provably optimal in the noisy rank one case (as well as in many of the biological examples that follow), while Gaussian random matrices are harder. Note however, that the duality gap between the semidefinite relaxations and the optimal solution is very small in both cases, while our bounds based on greedy solutions are not as good. This means that solving the relaxations in (\ref{eq:dual-y}) and \cite{dasp04a} could provide very tight upper bounds on sparse eigenvalues of random matrices. However, solving these semidefinite programs for very large values of $n$ remains a significant challenge.

\begin{figure}
\begin{center}
\psfrag{card}[t][b]{Cardinality}
\psfrag{var}[b][t]{Max. Eigenvalue}
\includegraphics[scale=.39]{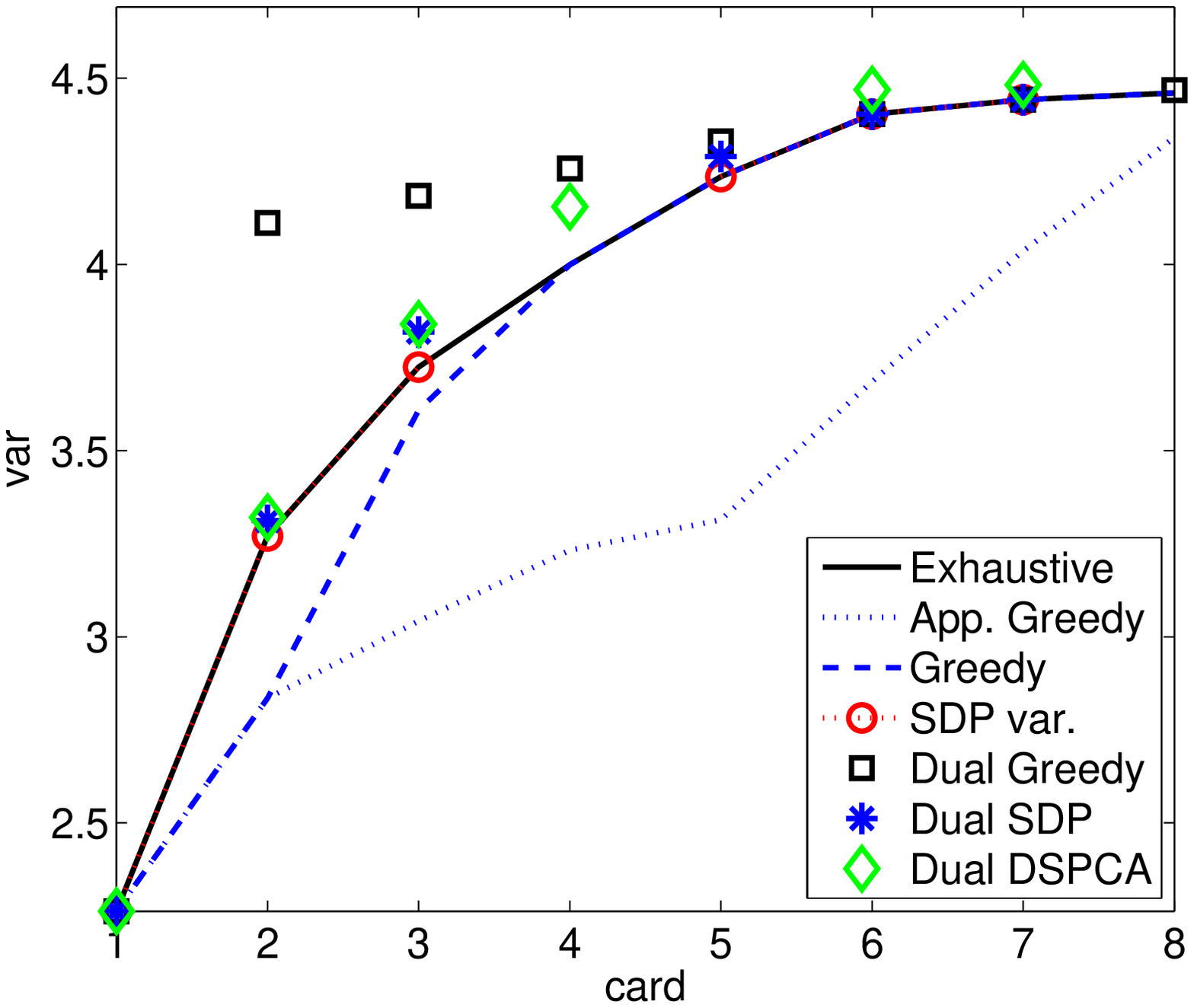}
\includegraphics[scale=.39]{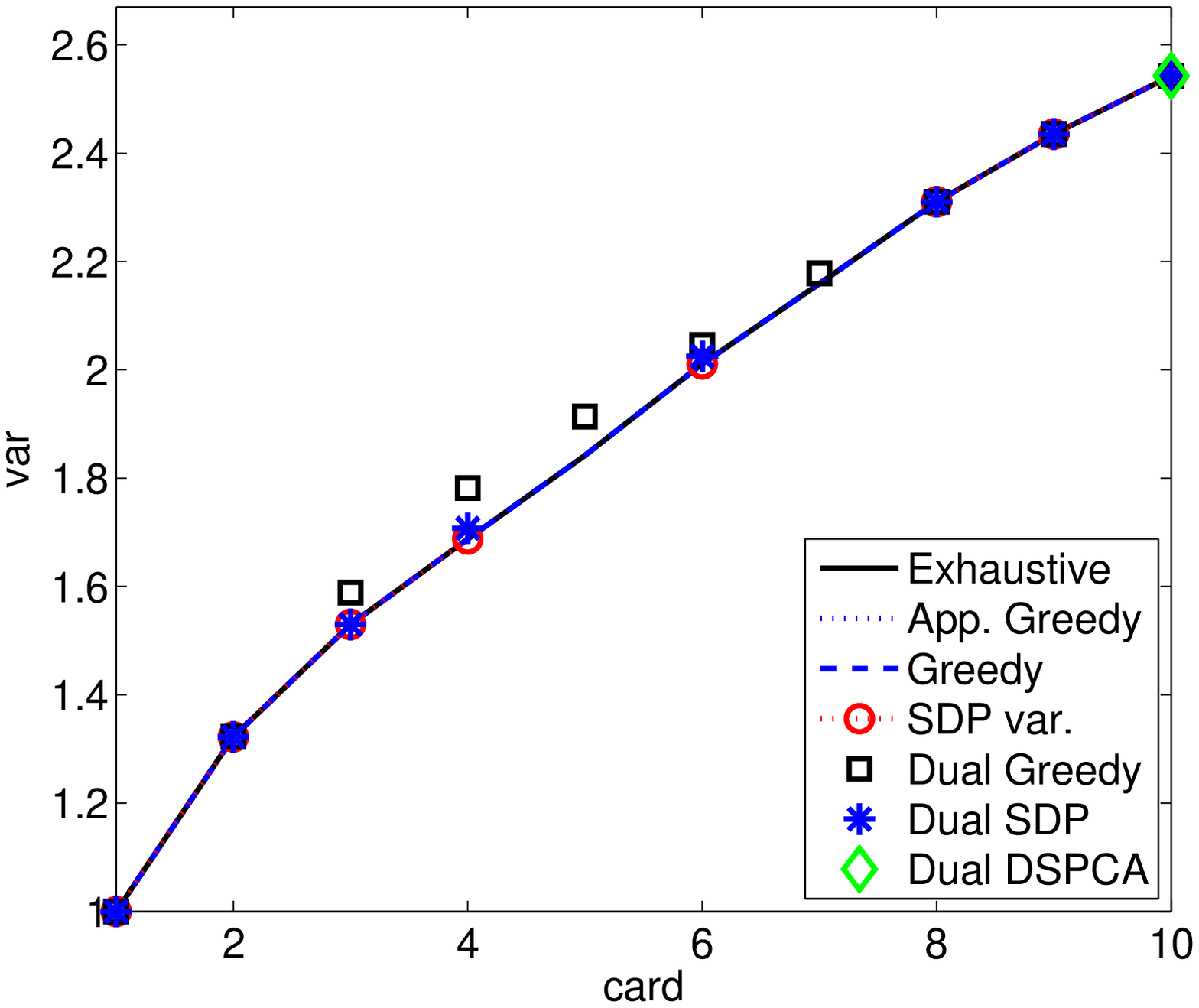}
\caption{Upper and lower bound on sparse maximum eigenvalues. We plot the maximum sparse eigenvalue versus cardinality, obtained using exhaustive search (solid line), the approximate greedy (dotted line) and fully greedy (dashed line) algorithms. We also plot the upper bounds obtained by minimizing the gap of a rank one solution (squares), by solving the semidefinite relaxation explicitly (stars) and by solving the DSPCA dual (diamonds). \emph{Left:} On a matrix $F^TF$ with $F$ Gaussian. \emph{Right:} On a sparse rank one plus noise matrix.\label{fig:recov}}
\end{center}
\end{figure}

\subsection{Biological Data}
We run the algorithm of \mysec{greedy-path} on two gene expression data sets, one on Colon cancer from \cite{Alon99}, the other on Lymphoma from \cite{Aliz00}. We plot the variance versus cardinality tradeoff curve in figure \ref{fig:eisen-bio}, together with the dual upper bounds from \mysec{eff-opt}. In both cases, we consider the 500 genes with largest variance. Note that for many cardinalities, we have optimal or very close to optimal solutions. In Table \ref{tab:rankedgenes}, we also compare the 20 most important genes selected by the second sparse PCA factor on the colon cancer data set, with the top 10 genes selected by the RankGene software by \cite{Su03}. We observe that 6 genes (out of an original 4027 genes) were both in the top 20 sparse PCA genes and in the top 10 Rankgene genes.
\begin{table}[th]
\begin{center}
\begin{tabular}{|p{7ex}|p{9ex}|p{7ex}|p{24ex}|}
\hline
Rank & Rankgene & GAN & Description \\ \hline
3&8.6 & J02854 & \footnotesize{Myosin regul.} \\\hline 
6&18.9 & T92451 & \footnotesize{Tropomyosin}\\ \hline 
7&31.5& T60155 & \footnotesize{Actin}  \\ \hline 
8&25.1 & H43887 & \footnotesize{Complement fact. D prec.} \\ \hline
10&2.1& M63391 & \footnotesize{Human desmin} \\ \hline 
12&32.3&T47377&\footnotesize{S-100P Prot.}\\
\hline
\end{tabular}
\caption{6 genes (out of 4027) that were both in the top 20 sparse PCA genes and in the top 10 Rankgene genes.}
\label{tab:rankedgenes}
\end{center}\end{table}

\begin{figure}[ht]
\begin{center}
\psfrag{card}[t][b]{Cardinality}
\psfrag{var}[b][t]{Variance}
\includegraphics[width=.7\textwidth]{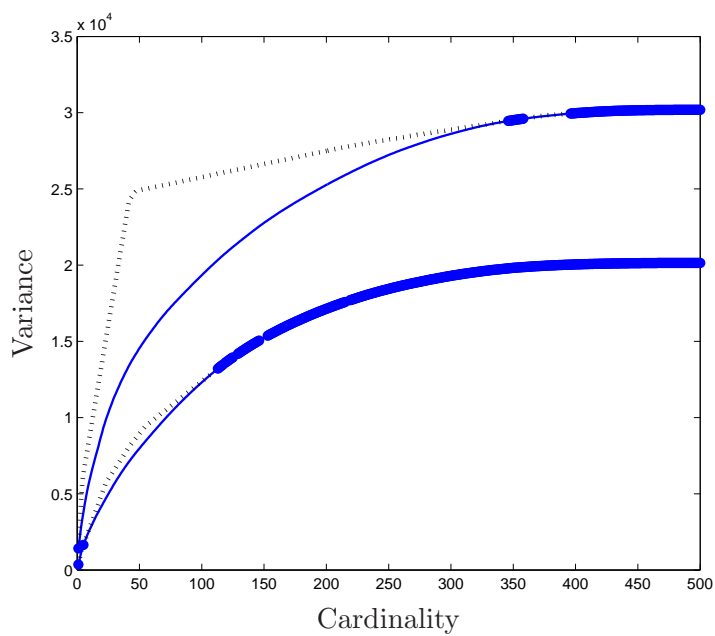} \caption{Variance (solid lines) versus cardinality tradeoff curve for two gene expression data sets, lymphoma (top) and colon cancer (bottom), together with dual upper bounds from \mysec{eff-opt} (dotted lines). Optimal points (for which the relative duality gap is less than $10^{-4}$) are in bold.\label{fig:eisen-bio}}
\end{center}
\end{figure}

\section{Conclusion}
We have presented a new convex relaxation of sparse principal component analysis, and derived tractable sufficient conditions for optimality. These conditions go together with efficient greedy algorithms that provide candidate solutions, many of which turn out to be optimal in practice. The resulting upper bounds also have direct applications to problems such as sparse recovery, subset selection or LASSO variable selection. Note that we extensively use this convex relaxation to test optimality and provide bounds on sparse extremal eigenvalues, but we almost never attempt to solve it numerically (except in some of the numerical experiments), which would provide optimal bounds. Having $n$ matrix variables of dimension $n$, the problem is of course extremely large and finding numerical algorithms to directly optimize these relaxation bounds would be an important extension of this work.

\vspace*{1cm}

\appendix

\section{Expansion of eigenvalues}
\label{app:perturbations}
In this appendix, we consider various results on expansions of eigenvalues we use in order to derive sufficient conditions. The following proposition derives a second order expansion of the set of eigenvectors corresponding to a single eigenvalue.

\begin{proposition}
\label{prop:pert}
Let $N \in \symm^n$. Let $\lambda_0$ be an eigenvalue of $N$, with multiplicity $r$ and eigenvectors $U \in \reals^{n \times r}$ (such that $U^T U  = \idm$). Let $\Delta$ be a matrix in $ \symm^n$.
If $\|\Delta\|_F$ is small enough, the matrix $N+\Delta$ has exactly $r$  (possibly equal) eigenvalues around $\lambda_0$ and if we denote by $(N+\Delta)_{\lambda_0}$ the projection of the matrix $N+\Delta$ onto that eigensubspace, we have:

\BEAS
\! (N+\Delta)_{\lambda_0}
\!&\!\! = \!\!&   \lambda_0 U U^T   + U U^T \Delta U U^T
+ \lambda_0 UU^T \Delta ( \lambda_0 \idm - N )^\dagger + 
\lambda_0  ( \lambda_0 \idm - N )^\dagger  \Delta  UU^T     \\
& &\!\!\!  \!\!\! \!\!\!  \!\!\!  + U U^T \Delta U U^T \Delta   ( \lambda_0 \idm - N )^\dagger 
+  ( \lambda_0 \idm - N )^\dagger  \Delta U U^T \Delta U U^T
+ U U^T \Delta  ( \lambda_0 \idm - N )^\dagger  U U^T  \\
& & \!\!\!  \!\!\!\!\!\!  \!\!\!  + \lambda_0 U U^T \Delta  ( \lambda_0 \idm - N )^\dagger \Delta   ( \lambda_0 \idm - N )^\dagger 
+  \lambda_0 ( \lambda_0 \idm - N )^\dagger  \Delta ( \lambda_0 \idm - N )^\dagger \Delta  UU^T
\\
& & \!\!\!  \!\!\! \!\!\!  \!\!\!  
+ \lambda_0 ( \lambda_0 \idm - M )^\dagger \Delta U U^T \Delta  ( \lambda_0 \idm - M )^\dagger + O(\|\Delta\|_F^3)
\EEAS
which implies the following expansion for the sum of the $r$ eigenvalues in the neigborhood of $\lambda_0$:
\BEAS
\Tr (N+\Delta)_{\lambda_0}   & = &  r \lambda_0 + \Tr U^T \Delta U + \Tr
U^T \Delta  ( {\lambda_0} \idm - N )^\dagger \Delta U  \\
& & +
{\lambda_0} \Tr ( {\lambda_0} \idm - N )^\dagger \Delta U U^T \Delta  ( {\lambda_0} \idm - N )^\dagger + O(\|\Delta\|_F^3).
\EEAS

\end{proposition}
\begin{proof}
We use the Cauchy residue formulation of projections on principal subspaces~\citep{kato}: given a symmetric matrix $N$, and a simple closed curve $\mathcal{C}$ in the complex plane that does
 not go through any of the eigenvalues of $N$, then
$$\Pi_\mathcal{C}(N) = \frac{1}{2i\pi} \oint_\mathcal{C} \frac{ d\lambda}{\lambda \idm - N}$$ is equal to the orthogonal projection onto the orthogonal sum
of all eigensubspaces of $N$ associated with eigenvalues in the interior of $\mathcal{C}$~\citep{kato}. This is easily seen  by writing down the eigenvalue decomposition $
N = \sum_{i=1}^n \lambda_i u_i u_i^T$, and the Cauchy residue formula ($\frac{1}{2i\pi} \oint_\mathcal{C} \frac{ d\lambda}{\lambda - \lambda_i} = 1$ if $\lambda_i$
is in the interior ${\rm int} ( \mathcal{C})$ of $\mathcal{C}$ and $0$ otherwise), and:
$$
\frac{1}{2i\pi} \oint_\mathcal{C} \frac{ d\lambda}{\lambda \idm - N}
= \sum_{i=1}^{n }  {u}_i  {u}_i^T  \times 
\frac{1}{2i\pi} \oint_\mathcal{C} \frac{ d\lambda}{\lambda  - \lambda_i } =
\sum_{ i , \ \lambda_i \in  {\rm int} ( \mathcal{C})} u_i u_i^T.
$$
See~\citet{rudin} for an introduction to complex analysis and Cauchy residue formula.
Moreover, we can obtain the restriction of $N$ onto a specific sum of eigensubspaces as:
$$
N \Pi_\mathcal{C}(N) = 
\frac{1}{2i\pi} \oint_\mathcal{C} \frac{ N d\lambda}{\lambda \idm - N}
=  
\frac{1}{2i\pi} \oint_\mathcal{C} \frac{\lambda d\lambda}{\lambda \idm - N}.
$$
From there we can easily compute expansions around a given $N$ by using expansions
of $(\lambda \idm -  N)^{-1}$. The proposition follows by considering a circle around $\lambda_0$ that is small enough to exclude other eigenvalues of $N$, and applying several times the Cauchy residue formula.
\end{proof}

We can now apply the previous proposition to our particular case:
\begin{lemma}
\label{lem:expansion}
For any $a \in \reals^n$, $\rho>0$ and $ B  = aa^T - \rho \idm$, we consider the function $F : X \mapsto \Tr(X^{1/2} B X^{1/2})_+$
from $\symm_+^n$ to $\reals$.  let $x \in \reals^n$ such that $\|x\|=1$. Let  $ Y \succeq 0 $. If $ x^T B x > 0$, then 
$$
F(  (1-t) xx^T+tY )  =   x^T B x  + \frac{t}{x^T B x} \Tr  B x  x^T B (Y-xx^T) + O(t^{3/2}),
$$
while if $ x^T B x < 0 $, then
$$
F(  (1-t) xx^T +tY )  =   \Tr \left(Y^{1/2} \left( B - \frac{B xx^T B}{x^T B x} \right) Y^{1/2} \right)_+ + O(t^{3/2}).
$$
\end{lemma}
\begin{proof}
We consider $X(t) = (1-t)xx^T+tY$. We have $X(t) = U(t) U(t)^T$ with 
$ U(t) = \left( \begin{array}{c} (1-t)^{1/2} x  \\   t^{1/2} Y^{1/2} \\ 
\end{array} \right)
$, which implies that the non zero eigenvalues of $X(t)^{1/2} B X(t)^{1/2}$ are the same as
the non zero eigenvalues of $U(t)^T B U(t)$. We thus have
$$F(X(t)) = \Tr ( M(t))_+,$$ with
\BEAS M(t) \!\!\!& = &  \!\!\!\left( \begin{array}{cc} (1-t) x^T B x & t^{1/2} (1-t)^{1/2} x^T B Y^{1/2} \\ 
t^{1/2} (1-t)^{1/2}y^T B x & t Y^{1/2} B Y^{1/2} \end{array} \right)  \\
& = & \!\!\! \left( \begin{array}{cc} x^T B x & 0 \\ 
0 & 0 \end{array} \right)
+ t^{1/2} \left( \begin{array}{cc}0 & \!\!\! x^T B Y^{1/2} \!\! \\ 
Y^{1/2} B x \!\!\! & 0 \end{array} \right)  + t
\left( \begin{array}{cc}\!\! - x^T B x & 0 \\ 
0 &\!\!\! Y^{1/2} B Y^{1/2} \!\! \end{array} \right) +O(t^{3/2}) \\
& = & M(0) + t^{1/2} \Delta_1 + t \Delta_2 +O(t^{3/2}).
\EEAS
The matrix $M(0)$ has a single (and simple) non zero eigenvalue which is equal to
$\lambda_0 = x^T B x$  with eigenvector $U =(1,0)^T$. The only other eigenvalue of $M(0)$ is zero, with multiplicity $n$. Proposition~\ref{prop:pert} can be applied to the two eigenvalues of $M(0)$: there is one eigenvalue of $M(t)$ around $x^T B x$, while the $n$ remaining ones are around zero.
The eigenvalue close to $\lambda_0$ is equal to:
\BEAS
\Tr ( M(t) )_{\lambda_0}
& = & t \Tr U^\top \Delta_2 U + \lambda_0 + t\Tr
U^T \Delta_1  ( {\lambda_0} \idm - M(0) )^\dagger \Delta_1 U  \\
& & +
{\lambda_0} \Tr ( {\lambda_0} \idm - M(0) )^\dagger \Delta_1 U U^T \Delta_1( {\lambda_0} \idm - M(0) )^\dagger + O(t^{3/2}) \\
& = &  
x^T B x   + \frac{t}{x^T B x} \Tr Bxx^T B ( Y - xx^T) +O(t^{3/2}).
\EEAS

For the remaining eigenvalues, we get that the projected matrix on the eigensubspace of $M(t)$ associated with eigenvalues around zero is equal to
\BEAS
 ( M(t) )_{0} & = & 
       t (\idm - U U^T ) \Delta_2 (\idm - U U^T ) 
+ t (\idm - U U^T )  \Delta_1  (  - M(0)  )^\dagger  (\idm - U U^T ) 
+ O(t^{3/2})
 \\
 & = & \left( \begin{array}{cc} 0 & 0 \\ 
0 & t  Y^{1/2} (  B - \frac{ Bx  x^T B }{x^T B x} ) Y^{1/2} \end{array} \right),
\EEAS
which leads to a positive part equal to $ t_+ \Tr   \left [
Y^{1/2} (  B - \frac{ Bx  x^T B }{x^T B x})  Y^{1/2}\right]_+$. When $x^T B  x >0$, then the matrix
is negative definite (because $ B= aa^T - \rho \idm$), and thus the positive part is zero. By summing the two contributions, we obtain the desired result.
\end{proof}

\bibliography{fullpathPCA}
\end{document}